\documentclass{article}

\usepackage{microtype}
\usepackage{graphicx}
\usepackage{subcaption}
\usepackage{pifont}       
\usepackage{amsfonts}       
\usepackage{nicefrac}       
\usepackage{microtype}      
\usepackage{xcolor}         
\usepackage{booktabs} 
\definecolor{mydarkblue}{rgb}{0,0.08,0.45}
\definecolor{mydarkred}{rgb}{0.6,0,0}
\definecolor{myblue}{HTML}{268BD2}
\definecolor{mygreen}{HTML}{658354}
\definecolor{orangeinplot}{HTML}{e29c7a}
\definecolor{purpleinplot}{HTML}{7676a4}
\definecolor{greeninplot}{HTML}{288308}

\usepackage{multirow}
\usepackage{makecell}
\usepackage{bbm}
\usepackage{booktabs}
\usepackage{pifont}
\newcommand{\cmark}{\ding{51}} 
\newcommand{\xmark}{\ding{55}} 

\usepackage{float}
\usepackage[colorlinks,
linkcolor=mydarkred,
citecolor=myblue]{hyperref}
\usepackage{xspace}
\newcommand{\model}{\textsc{ARS}\xspace}
\usepackage{hyperref}

\usepackage{algorithm}
\usepackage[ruled,vlined,algo2e]{algorithm2e}

\usepackage[accepted]{icml2026}
\usepackage[table]{xcolor} 
\usepackage{amsmath}
\usepackage{amssymb}
\usepackage{mathtools}
\usepackage{amsthm}
\usepackage{bm}     
\newcommand{\ba}{\bm{a}}
\newcommand{\bx}{\bm{x}}
\newcommand{\br}{\bm{r}}
\newcommand{\bh}{\bm{h}}
\newcommand{\bz}{\bm{z}}
\newcommand{\bu}{\bm{u}}
\newcommand{\bb}{\bm{b}}
\usepackage[capitalize,noabbrev]{cleveref}

\theoremstyle{plain}
\newtheorem{theorem}{Theorem}[section]
\newtheorem{proposition}[theorem]{Proposition}
\newtheorem{lemma}[theorem]{Lemma}

\theoremstyle{definition}
\newtheorem{definition}[theorem]{Definition}
\newtheorem{assumption}[theorem]{Assumption}
\theoremstyle{remark}

\usepackage[textsize=tiny]{todonotes}

\icmltitlerunning{Harnessing Reasoning Trajectories for Hallucination Detection via Answer-agreement Representation Shaping}

\begin{document}

\twocolumn[
  \icmltitle{Harnessing Reasoning Trajectories for Hallucination Detection via Answer-agreement Representation Shaping}



  \icmlsetsymbol{equal}{*}

  \begin{icmlauthorlist}
    \icmlauthor{Jianxiong Zhang}{ntu,scu}
    \icmlauthor{Bing Guo}{scu}
    \icmlauthor{Yuming Jiang}{scu}
    \icmlauthor{Haobo Wang}{zju}
    \icmlauthor{Bo An}{ntu}
    \icmlauthor{Sean Du}{ntu}

  \end{icmlauthorlist}

  \icmlaffiliation{scu}{College of Computer Science, Sichuan University, China}
   \icmlaffiliation{zju}{School of Software Technology,
Zhejiang University, China}
  \icmlaffiliation{ntu}{College of Computing and Data Science, Nanyang Technological University, Singapore}

  \icmlcorrespondingauthor{Sean Du}{xuefeng.du@ntu.edu.sg}

  \icmlkeywords{Machine Learning, ICML}

  \vskip 0.3in
]



\printAffiliationsAndNotice{}  

\begin{abstract}
Large reasoning models (LRMs) often generate long, seemingly coherent reasoning traces yet still produce incorrect answers, making hallucination detection challenging. Although trajectories contain useful signals, directly using trace text or vanilla hidden states for detection is brittle: traces vary in form and detectors can overfit to superficial patterns rather than answer validity. We introduce Answer-agreement Representation Shaping (\model), which learns detection-friendly trace-conditioned representations by explicitly encoding answer stability. \model generates counterfactual answers through small latent interventions, specifically, perturbing the trace-boundary embedding, and labels each perturbation by whether the resulting answer agrees with the original. It then learns representations that bring answer-agreeing states together and separate answer-disagreeing ones, exposing latent instability indicative of hallucination risk. The shaped embeddings are plug-and-play with existing embedding-based detectors and require no human annotations during training. Experiments demonstrate that \model consistently improves detection and achieves substantial gains over strong baselines.  Code is available at: \url{https://github.com/radiolab-ntu/ars_icml2026}.
\end{abstract}

\section{Introduction}
Language models are increasingly deployed as reasoning-centric systems: they generate intermediate reasoning traces and then produce a final answer in domains such as multi-hop QA~\cite{chen2024llm}, math~\cite{huan2025does}, and tool-augmented decision making~\cite{qin2024toolllm}. Despite rapid progress, a persistent reliability failure is that models can generate answers that are fluent and seemingly well-justified, yet factually incorrect, which are commonly referred to as hallucinations~\cite{huang2025survey}. This problem is amplified for large reasoning models (LRMs)~\cite{hou2025t1advancinglanguagemodel,deepseekai2025deepseekr1incentivizingreasoningcapability,yang2025qwen3technicalreport}: a single unsupported intermediate step can propagate through a long trajectory and culminate in a confident but wrong answer, while the surface form of the reasoning trace can remain persuasive~\cite{yao2025reasoning}.

\begin{figure}[t]
  \begin{center}
    \centerline{\includegraphics[width=\columnwidth]{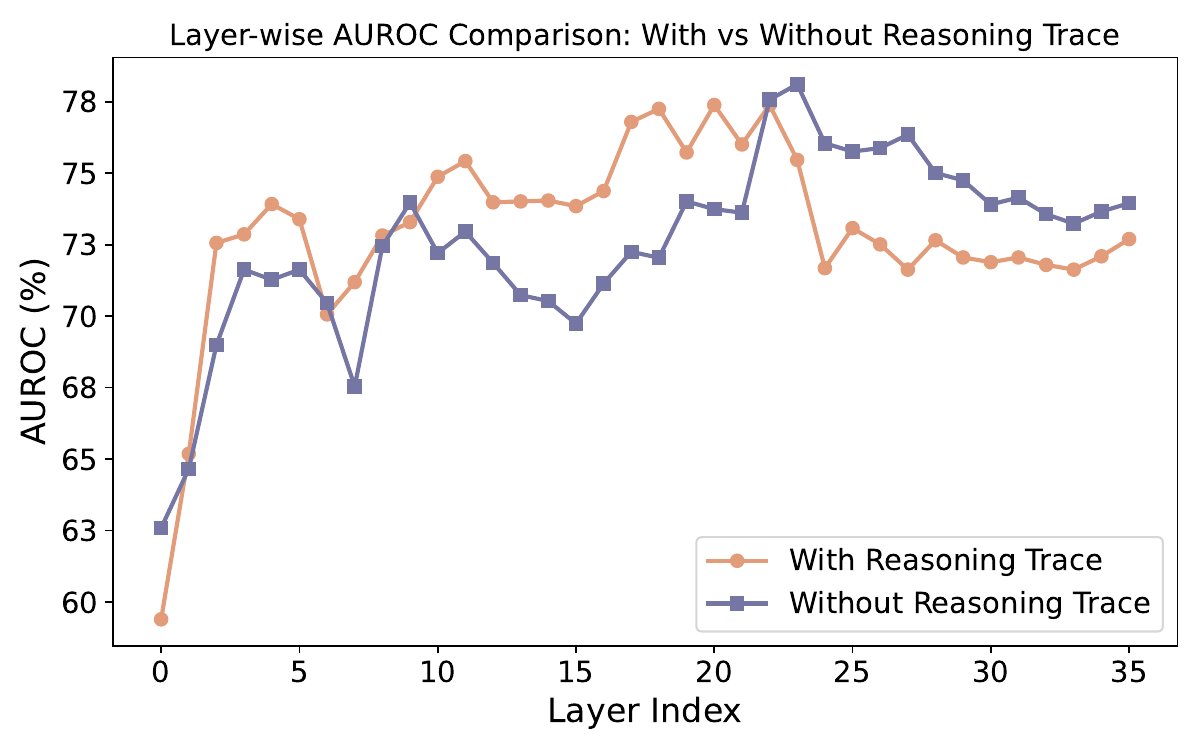}}
    \caption{\small 
   {\textbf{Effect of reasoning trajectories on hallucination detection in LRMs}.} We compare  detection performance for the same LRM (Qwen3-8B~\cite{yang2025qwen3technicalreport}) with and without an explicit reasoning trajectory, using representations extracted from each layer for the same answers. Consistent with our hypothesis, reasoning traces can sometimes obscure answer-level hallucination signals. The dataset is TruthfulQA~\cite{lin2022truthfulqa}.
    }
    \label{linear-probe}
  \end{center}
\end{figure}

A natural direction is to use the reasoning trajectory as a richer signal for hallucination detection. However, leveraging trajectories in practice is non-trivial. First, reasoning traces are not uniquely determined: the same prompt can admit multiple plausible traces, and LRMs may vary intermediate steps while keeping the \emph{answer} unchanged. Second, hallucination is ultimately an \emph{answer-level} property, yet trajectories span many tokens and layers, where irrelevant stylistic variation can dominate representation-based scores. As a result, naively probing hidden states along the full trace can be brittle: the detector may overfit to superficial trace patterns, or miss cases where the model is internally “close” to changing the answer. Consistent with this, Figure~\ref{linear-probe} shows that straightforward probing yields mixed, and sometimes worse detection performance when reasoning trajectories are included. These observations suggest that the key is not merely to \emph{use} trajectories, but to \emph{distill} from them an answer-centric signal that is stable and detection-friendly.

In this work, we investigate an important yet underexplored research question:
\begin{center}
   \textbf{\textit{Can we leverage the reasoning trajectory  to shape detection-friendly answer representations?}}
\end{center}

To address this, we introduce \textbf{A}nswer-agreement \textbf{R}epresentation \textbf{S}haping (\model), a novel learning framework that optimizes \emph{trace-conditioned answer embeddings} explicitly organized by \emph{answer agreement} under small internal interventions. The key idea in \model is to generate counterfactual answers by applying small perturbations directly to the model's hidden state \emph{at the trace boundary}, which is the last-token embedding of the reasoning trace at the penultimate layer. We then continue decoding to obtain an alternative final answer. Intuitively, if the model's internal state supports a truthful answer with a large margin, small perturbations should rarely change the final answer; conversely, hallucinated answers are often supported by fragile internal states, where small perturbations can redirect decoding toward inconsistent answers.

Crucially, rather than constructing the shaping signal by editing text which  requires careful perturbation design, \model uses latent-state interventions to obtain paired examples tied to the model’s own decision geometry. Starting from a given prompt and reasoning trace, we perturb the penultimate-layer state at the trace boundary and decode a counterfactual answer. We form {positive} pairs when the answer agrees with the original answer, and {negative} pairs when it disagrees. Using these automatically constructed pairs, \model optimizes a lightweight mapping on answer representations that pulls agreement pairs together and pushes disagreement pairs apart, yielding embeddings that more directly reflect answer stability.

Once trained, \model produces a shaped {trace-conditioned answer embedding} that can be scored by a range of embedding-based detectors, such as supervised and unsupervised probing~\cite{azaria2023internal,burns2022discovering}, subspace scoring~\cite{du2024haloscope}, and eigen-based scores~\cite{chen2024inside}, without requiring expensive sampling at test time. Specifically, on a representative benchmark TruthfulQA, our learned trace-conditioned answer embeddings improve detection performance by 19.79\% compared to the vanilla LRM embeddings, and achieve state-of-the-art detection performance of 86.64\%.

\begin{figure*}[t]
  \begin{center}
    \centerline{\includegraphics[width=\textwidth]{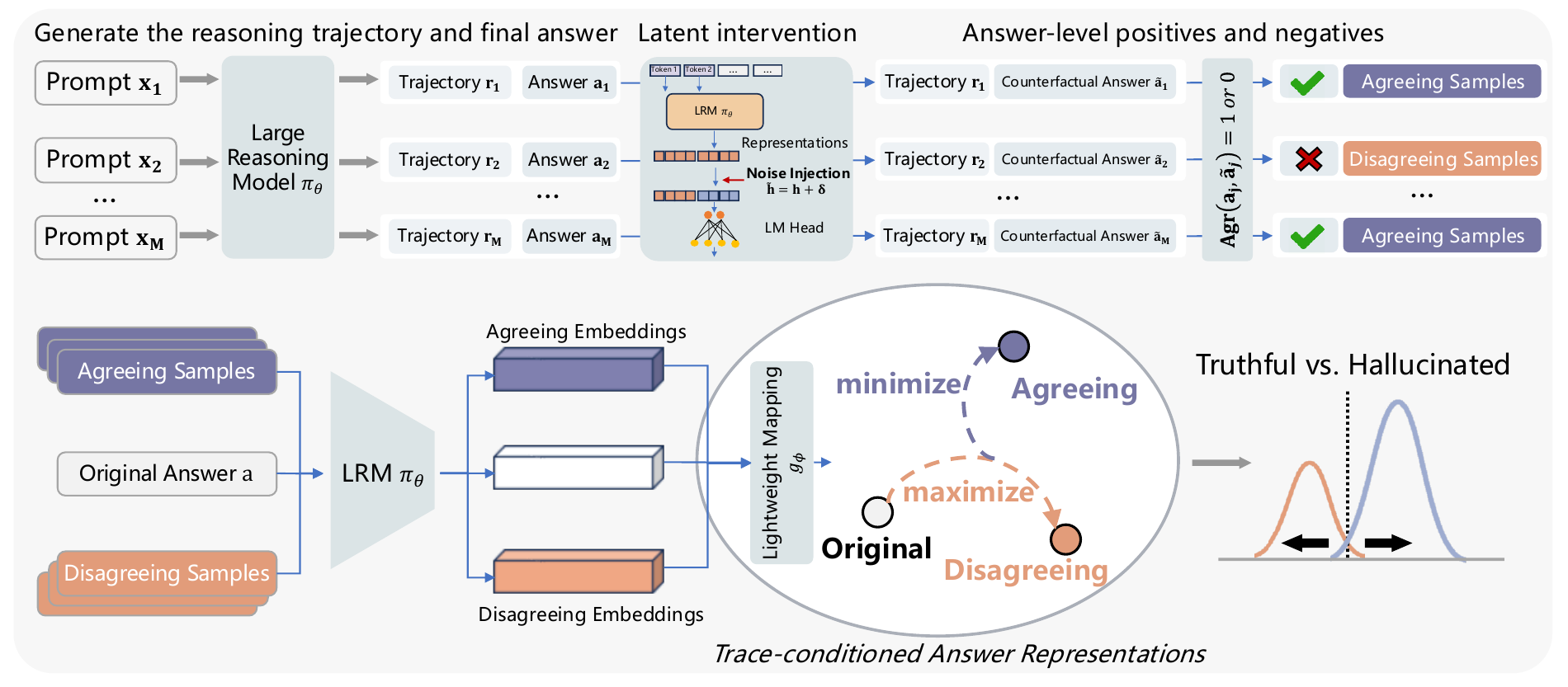}}
    \caption{\small 
    \textbf{Overview of \model framework for hallucination detection in LRMs}. \model firstly generates counterfactual answers by latent intervention at the trace boundary, and then learns a lightweight mapping that shapes trace-conditioned answer representations with an answer-agreement signal. This can make truthful vs. hallucinated outputs more separable for downstream embedding-based detectors. 
    }
    \label{main-figure}
  \end{center}
\end{figure*}

Our key contributions are summarized as follows:
\begin{itemize}
    \item To the best of our knowledge, \model is the first framework that enables principled use of reasoning traces to {shape answer representations} for hallucination detection in LRMs, by inducing counterfactual answers via latent perturbations at the end of the reasoning trace.
    \item We propose an agreement-driven representation shaping objective that organizes {final answer embeddings} by whether small internal changes preserve the answer, avoiding reliance on brittle text perturbations.
    \item Comprehensive  empirical (Section~\ref{scr:experiment}) and theoretical (Proposition~\ref{prop:agree_bound_main}) analyses are provided to understand when and why \model improves detection. The results provide insights into hallucination detection for frontier reasoning AI systems.  
\end{itemize}

\section{Related Work}
\paragraph{Hallucination detection} has attracted a surge of interest in recent
years~\cite{chen2023hallucination,sriramanan2024llmcheck,su2024unsupervised,zhang2025siren,zhou2025hademif}.  One line of work performs  detection
by devising post-hoc scoring functions, including logit-based
methods that utilize token-level probability as an uncertainty
score~\cite{malinin2021uncertainty,ren2023self,duan2024shifting}; consistency-based detectors that  assess uncertainty by
evaluating the consistency across multiple responses~\cite{kuhn2023semantic,manakul2023selfcheckgpt,chen2024inside,mundler2023self}; verbalized methods that prompt LLMs to express their uncertainty in human language~\cite{lin2022teaching,xiong2023can}; and representation-based approaches that extract hallucination signals from LLM embeddings~\cite{du2024haloscope}. Another line of work addressed
detection by training a classifier~\cite{azaria2023internal,burns2022discovering,kuhn2023semantic,park2025steer}, which usually requires additional human annotation.

While effective for standard LLMs, these approaches may not be well suited for reasoning models, which typically generate long-horizon thinking trajectories before the final answer.~\citet{cheng2025chain} discovered that reasoning can obscure hallucination  signals across different uncertainty estimation methods. The most relevant works that tackle hallucination detection for LRMs are~\cite{sun2025detection,anonymous2025unraveling,wang2025joint,lu2026streaminghallucinationdetectionlong}. None of them shape the LRM representations to improve detection and the comparison with  \model is in  Section~\ref{sec:mainresult}.

\paragraph{Large reasoning models} build upon the chain-of-thought paradigm~\cite{wei2022chain}, and generate explicit multi-step reasoning traces often via training recipes that incentivize intermediate thinking or long-horizon problem solving~\cite{deepseekai2025deepseekr1incentivizingreasoningcapability,yang2025qwen3technicalreport,hou2025t1advancinglanguagemodel}. This paradigm has improved downstream task accuracy, but it also introduces new reliability challenges of model hallucination~\cite{yao2025reasoning}.  There are several works that study the intermediate reasoning trajectory by step-wise checking and process-level supervision~\cite{cobbe2021gsm8k,lightman2023let}.  Our work is complementary: rather than judging the textual validity of each step, \model treats the trace primarily as a \emph{conditioning context} and shapes the trace-conditioned answer
embeddings for hallucination detection.

\section{Problem Setup}
\label{sec:method-setup}
Formally, we describe the reasoning model generation and the problem of hallucination detection.

\paragraph{LLM generation with reasoning trajectories.}
We consider an $L+1$-layer causal language model $\pi_\theta$ that, given a prompt
$\bx=\{x_1,\ldots,x_n\}$, generates an output sequence autoregressively.
In reasoning-centric settings, the output is naturally decomposed into a
{reasoning trajectory} $\br=\{x_{n+1},\ldots,x_{n+t}\}$ followed by a
{final answer} $\ba=\{x_{n+t+1},\ldots,x_{n+m}\}$, where $t<m$ may vary across examples.
At each decoding step $i>n$, the model defines a distribution over the vocabulary
$\mathcal{V}$ conditioned on the prefix $\{x_1,\ldots,x_{i-1}\}$:
\begin{equation}
    \pi_\theta(x_i \mid x_{<i})
    = \mathrm{softmax}\!\big(\mathbf{W}_o\, \bh_L(x_{<i}) + \bb_o\big),
\end{equation}
where $\bh_L(x_{<i})\in\mathbb{R}^d$ denotes the penultimate-layer hidden representation
associated with the next-token prediction, and $(\mathbf{W}_o,\bb_o)$ are the final-layer parameters ($L+1$-th layer).
For simplicity we assume greedy decoding in exposition, i.e.,
$x_i=\arg\max_{x\in\mathcal{V}} \pi_\theta(x\mid x_{<i})$, though our method applies to standard decoding variants.

\paragraph{Hallucination detection.}
Given a prompt and a model generation $(\bx,\br,\ba)$,
the goal of hallucination detection is to predict whether the final answer $\ba$
is truthful  under the task-specific criterion.
We write $y\in\{0,1\}$ for the (unknown) truthfulness label of $\ba$,
and seek a detector $G$ that maps the prompt, trajectory, and answer to a binary prediction:
\begin{equation}
    G(\bx,\br,\ba) \in \{0,1\},
\end{equation}
where $G(\bx,\br,\ba)=1$ indicates a truthful answer and $0$ indicates a hallucination.
Equivalently, one may view $G$ as producing a real-valued score
$s(\bx,\br,\ba)$ that is thresholded to obtain a decision.

\paragraph{Key challenge.}
Reasoning trajectories provide additional signal beyond the final answer, but they also exhibit substantial
surface-form variability. As a result, detectors that operate directly on trace text or vanilla hidden states
can be brittle: they may overfit to stylistic trace artifacts instead of features tied to whether the answer is stable and correct.
Our approach addresses this by shaping the LRM  representation that emphasizes \emph{answer agreement} under small latent perturbations,
which we introduce next.

\section{Proposed Framework: \model}
\label{sec:method}

In this paper, we propose a novel framework that enables {principled use of reasoning trajectories} for hallucination detection by \emph{shaping the trace-conditioned answer embeddings} with an \emph{answer-agreement} signal induced by small latent interventions. The central goal is to transform the vanilla LRM hidden state, which is often dominated by surface-form variability, into a detection-friendly representation that tracks whether the final answer is \emph{stable} under minor internal changes.

Our key motivation is that hallucinated answers are often supported by {instability}~\cite{chen2024inside}: small variations in the decoding process can lead to inconsistent answers. However, existing works usually realize this signal through {output-space} stochasticity, e.g., drawing multiple samples and measuring disagreement or entropy~\cite{kuhn2023semantic}, which incurs substantial test-time overhead and can be brittle to prompt format and paraphrasing. \model instead makes instability {explicit} in the shaped embeddings, ensuring that downstream hallucination detectors can more reliably perform hallucination detection.

\paragraph{Overview.} Given a prompt $\bx$, a reasoning trajectory $\br$, and a final answer $\ba$ produced by a fixed reasoning model $\pi_\theta$, \model learns a lightweight mapping
$g_\phi:\mathbb{R}^d\rightarrow\mathbb{R}^k$ that converts the vanilla answer representation $\bu$ into a shaped embedding $\bz$ used for downstream hallucination detection.
Our framework consists of two integral components:
\begin{itemize}
    \item \textit{Latent intervention for counterfactual answers.} We perturb the model's latent state \emph{at the trace boundary}—the last-token embedding of $\br$ at the penultimate layer and continue decoding to obtain counterfactual answers.
    \item \textit{Answer-agreement representation shaping.} We label each counterfactual by whether its answer \emph{agrees} with the original, and optimize $g_\phi$ so that agreeing states are mapped closer than disagreeing states, yielding embeddings that expose latent answer instability.
\end{itemize}
Notably, \model does \emph{not} require hallucination labels and updating the base LRM parameters $\theta$; it only trains a small head $g_\phi$, making it  lightweight and easy to integrate.

\subsection{Latent Intervention for Counterfactual Answers}
\label{sec:method1}
A key design choice is to intervene at a compact summary of the entire reasoning trajectory. Let
$\bx\!\oplus\!\br=\{x_1,\ldots,x_{n+t}\}$ be the prefix ending at the last token of the trace.
We denote the corresponding \emph{trace-boundary representation} as $ \bh \;=\;  \bh_L(\bx\!\oplus\!\br) \in \mathbb{R}^d,$ where $ \bh_L(\cdot)$ is the model representation at the penultimate layer (as in Section~\ref{sec:method-setup}).
Intuitively, $\bh$ captures the model's internal state {right before} answer generation begins, and thus provides a natural control point for probing answer stability.

We generate counterfactual answers by applying small perturbations to $\bh$ and resuming decoding.
Concretely, we draw a perturbation vector $\boldsymbol{\delta}\sim \mathcal{D}$ and construct
$\tilde{\bh} \;=\; \bh +  \boldsymbol{\delta},$ where $\mathcal{D}$ is a simple distribution (e.g., isotropic Gaussian noise $\mathcal{N}(0, \sigma^2\boldsymbol{I})$).
We then continue autoregressive decoding {starting from} the intervened boundary state to obtain a counterfactual answer $\tilde{\ba}$ by $ \tilde{\ba} \;=\; \mathrm{Decode}_\theta(\bx\!\oplus\!\br;\,\tilde{\bh}).$ In practice, for each $(\bx,\br,\ba)$ we sample $M$ perturbations to obtain a set of counterfactual answers
$\{\tilde{\ba}_j\}_{j=1}^M$.

\paragraph{Why intervene at the trace boundary?} Intervening in the mid-trace is suboptimal because the model has not yet committed to a concrete answer, so a small perturbation can arbitrarily reshape subsequent reasoning, mixing answer-relevant effects with large, noisy changes in trace form. Additionally, intervening after answer decoding is also less informative: once decoding has started, later tokens are heavily constrained by earlier answer tokens, so perturbations often cause only superficial edits rather than clean answer flips. The trace boundary is precisely where the model’s internal state has incorporated the full reasoning  trajectory while still having maximal freedom to determine the answer, making answer agreement under small perturbations a sharp and interpretable stability signal. Detailed empirical verification is provided in Section~\ref{sec:ablationstudy}.

\subsection{Answer-agreement Representation Shaping}
\label{sec:method2}
We now describe how \model converts the counterfactual answers induced by latent interventions (Section~\ref{sec:method1}) into a training signal that {shapes} the detection-friendly trace-conditioned answer representation.

\paragraph{Answer agreement as supervision.}
Given the original answer $\ba$ and a counterfactual answer $\tilde{\ba}$, we define an \emph{answer-agreement} indicator $ \mathrm{Agr}(\ba, \tilde{\ba}) \in \{0,1\},$ which returns $1$ if $\tilde{\ba}$ is considered equivalent to $\ba$ and $0$ otherwise~\footnote{$\mathrm{Agr}(\cdot,\cdot)$ can  be practically instantiated via textual similarity metrics or LRM judge, thus requiring no gt. labels $y$. }.

\paragraph{Answer-level positives and negatives.}
For each generation $(\bx, \br, \ba)$, Section~\ref{sec:method1} yields $M$  counterfactual answers $\{\tilde{\ba}_j\}_{j=1}^M$.
\model forms positives and negatives at the \emph{answer-representation} level.
Let $\tilde{\bu}$ denote the vanilla {answer embedding} extracted from the frozen LRM for trajectory and answer $(\bx \oplus \br, \tilde{\ba})$, we partition counterfactual \emph{answer embeddings} into agreement \& disagreement sets:
\begin{align}
\label{eq:posneg-ans}
    \mathcal{U}^{+}(\bx,\br,\ba)
    &= \{\tilde{\bu}_j\;:\; \mathrm{Agr}(\ba,\tilde{\ba}_j)=1\}, \\
    \mathcal{U}^{-}(\bx,\br,\ba)
    &= \{\tilde{\bu}_j\;:\; \mathrm{Agr}(\ba,\tilde{\ba}_j)=0\}.
\end{align}
Intuitively, $\mathcal{U}^{+}$ collects alternative internal realizations that lead to the {agreeing} answer, while $\mathcal{U}^{-}$ collects realizations that lead to a disagreeing answer.
Our hypothesis is that hallucinated answers may exhibit a larger nearby region that maps into $\mathcal{U}^{-}$, reflecting a smaller internal stability margin.

\paragraph{Shaping embeddings by answer agreement.}
\model learns a lightweight mapping $g_\phi:\mathbb{R}^d\!\rightarrow\!\mathbb{R}^k$ that transforms the vanilla answer embedding into a shaped representation:
$
    \bz = g_\phi(\bu).
$
Throughout, the base model parameters $\theta$ remain fixed; only $\phi$ is optimized.

Specifically, we optimize $g_\phi$ so that agreement-preserving embeddings concentrate while disagreement embeddings are pushed apart.
 For each anchor $\bz$ (original answer) we sample one agreeing embedding $\tilde{\bz}^+\sim \{g_\phi(\tilde{\bu}^+): \tilde{\bu}^+\in\mathcal{U}^+\}$ and treat the set of disagreement embeddings $\mathcal{Z}^-=\{g_\phi(\tilde{\bu}^-):\tilde{\bu}^-\in\mathcal{U}^-\}$ as competing alternatives.
We minimize the following objective:
\begin{equation}
\label{eq:infoNCE}
\small
\mathcal{L}_{\model}
=
-\frac{\mathrm{sim}(\bz,\tilde\bz^+)}{\tau}
+
\log \sum_{\tilde\bz' \in \{\tilde\bz^+\}\cup \mathcal Z^-}
\exp\!\Big(\frac{\mathrm{sim}(\bz,\tilde\bz')}{\tau}\Big),
\end{equation}

where $\mathrm{sim}(\cdot,\cdot)$ is cosine similarity, and $\tau>0$ is a temperature.
This objective explicitly increases the relative similarity between the original answer embedding and an answer-agreeing embedding, while decreasing similarity to answer-disagreeing embeddings, thereby shaping the LRM embeddings that directly reflect answer stability. Our framework \model is summarized in Algorithm~\ref{alg:model}.

\begin{table*}[!t]
\centering
\caption{\small \textbf{Comparison with competitive hallucination detection methods on different datasets}. ``Single sampling'' indicates whether the approach requires multiple generations during inference. ``Supervision'' indicates whether the approach requires ground truth annotation during training or testing. All values are percentages (AUROC). We present the results of \model with CCS and supervised probing as downstream detectors. The best results are highlighted in \textbf{bold}.}
\label{tab:all-detectors}
\scalebox{0.75}{ \begin{tabular}{c l c c cccc}
\toprule
\textbf{Model} & \textbf{Method} &  \textbf{Single Sampling}&   \textbf{Supervision}& \textbf{TruthfulQA} & \textbf{TriviaQA} & \textbf{GSM8K} & \textbf{MATH-500} \\
\midrule
\multirow{12}{*}{Qwen3-8B}
& Perplexity~\cite{ren2022out}       & \cmark & \xmark &59.72  &53.00  &60.80  &51.62  \\
& Semantic Entropy~\cite{kuhn2023semantic}       & \xmark   & \xmark&65.60   &58.37  &72.51  &56.13  \\
& Lexical Similarity~\cite{lin2023generating}    &   \xmark &  \xmark&58.81   &62.03  &66.38  &44.13  \\
& SelfCKGPT~\cite{manakul2023selfcheckgpt}     &  \xmark & \xmark & 52.15 &53.84  &54.33  &55.47  \\
& Verbalized Certainty~\cite{lin2022teaching}& \cmark&\xmark & 45.37  &35.89  &43.27  &23.87  \\
&TSV~\cite{park2025steer}       & \cmark &\cmark &77.08  &89.67  &83.15  &63.12  \\
\cmidrule(lr){2-8}
& \multicolumn{7}{c}{\textit{LRM-based}} \\
& RHD~\cite{sun2025detection}   &   \cmark & \xmark &56.14   &56.53  &57.60  &50.51  \\
& RACE~\cite{wang2025joint}      &  \xmark & \xmark&67.57   &86.57  &72.55  &63.02  \\
& G-Detector~\cite{anonymous2025unraveling}       &  \cmark & \cmark&71.86   &90.52  &83.78  &57.67  \\

\rowcolor{gray!10}\cellcolor{white}& \textbf{\model (CCS)}    &   \cmark& \xmark & \textbf{86.64}  &88.54  &\textbf{90.37}  &\textbf{78.66}  \\
\rowcolor{gray!10}\cellcolor{white}& \textbf{\model (Probing) }     & \cmark &  \cmark  & {83.66} &\textbf{91.62}  &89.88  &78.17  \\
\midrule
\multirow{12}{*}{\shortstack{DeepSeek-R1-\\Distill-Llama-8B}}
& Perplexity~\cite{ren2022out}      &\cmark   &\xmark&56.62  &48.56  &58.48  &40.96  \\
& Semantic Entropy~\cite{kuhn2023semantic}      & \xmark & \xmark&55.47   &49.97  &61.98  &43.60  \\
& Lexical Similarity~\cite{lin2023generating}    &   \xmark& \xmark &58.64   &50.27  &56.01  &49.92  \\
& SelfCKGPT~\cite{manakul2023selfcheckgpt}    &  \xmark& \xmark &55.95  &50.33  &50.58  &59.15  \\
& Verbalized Certainty~\cite{lin2022teaching} &\cmark&\xmark &50.00   &49.88  &50.00  &48.98  \\
& TSV~\cite{park2025steer}  &    \cmark &  \cmark &69.49  &85.73  &\textbf{78.29} &63.24  \\
\cmidrule(lr){2-8}
& \multicolumn{7}{c}{\textit{LRM-based}} \\
& RHD~\cite{sun2025detection}     &    \cmark & \xmark&56.64   &51.07  &61.67  &56.50  \\
& RACE~\cite{wang2025joint}         & \xmark&\xmark &62.44  &49.94  &68.59  &53.55  \\
& G-Detector~\cite{anonymous2025unraveling}      & \cmark & \cmark&70.01   &52.25  &70.38  &64.45  \\

\rowcolor{gray!10}\cellcolor{white}& \textbf{\model (CCS) }     & \cmark & \xmark   & \textbf{80.89} &\textbf{88.86}  &74.72  &\textbf{86.38}  \\
\rowcolor{gray!10}\cellcolor{white}& \textbf{\model (Probing) }     & \cmark &\cmark    & {76.98} &87.45  &{77.62}  &79.95  \\
\bottomrule
\end{tabular}}
\end{table*}

\paragraph{Mathematical interpretation.}
We provide a simple analysis connecting our answer-agreement  objective to hallucination detection. Specifically, we show that \model-shaped embeddings can achieve a bounded hallucination detection error (evaluated by supervised probing~\cite{azaria2023internal}), where the bound is relevant to our shaping objective. Firstly, we define:

\begin{definition}[Answer stability score]
\label{def:alpha}
Given a generation $(\bx,\br,\ba)$ from a frozen LRM $\pi_\theta$, let
$\bh := \bh_L(\bx\oplus\br)$ denote the trace-boundary representation.
For a perturbation $\boldsymbol{\delta}\sim\mathcal{D}$, define the counterfactual answer
$\tilde{\ba}(\boldsymbol{\delta}):=\mathrm{Decode}_\theta(\bx\oplus\br;\bh+\boldsymbol{\delta})$.
The \emph{answer stability score} is
\begin{equation}
\label{eq:alpha_def_main}
\alpha(\bx,\br,\ba)
:= \Pr_{ \boldsymbol{\delta}\sim\mathcal{D}}\!\left[
\mathrm{Agr}\big(\ba,\tilde{\ba}( \boldsymbol{\delta})\big)=1
\right]\in[0,1].
\end{equation}
\end{definition}

Recall that \model shapes embeddings so that an \emph{agreement} sample $\tilde{\bz}^+$ is closer to the anchor $\bz$
than a \emph{disagreement} sample $\tilde{\bz}^-$. Define the agreement-separation indicator for one triple $(\bz,\tilde{\bz}^+,\tilde{\bz}^-)$:
\begin{equation}
\label{eq:sep_def}
\mathsf{Sep}(\bz,\tilde{\bz}^+,\tilde{\bz}^-)
:= \mathbf{1}\!\left\{\mathrm{sim}(\bz,\tilde{\bz}^+) \ge \mathrm{sim}(\bz,\tilde{\bz}^-)\right\}.
\end{equation}
Let $\eta_\phi := \Pr\big[\mathsf{Sep}(\bz,\tilde{\bz}^+,\tilde{\bz}^-)=1\big]$ denote the probability of agreement separation under our construction
(Section~\ref{sec:method2}), and we have the following:

\begin{proposition}
\label{prop:agree_bound_main}
(Informal.) Let $y\in\{0,1\}$ be the truthfulness label. Define
$e_\alpha := \inf_T \Pr\!\left(\mathbf{1}\{\alpha\ge T\}\neq y\right),$
the best achievable hallucination detection error if the stability score $\alpha$ were observed.
There exists a constant $C >0$ and a detector $\hat{y}$ computable from a supervised probe on \model's shaped embeddings such that
\begin{equation}
\label{eq:agree_bound_main}
\Pr(\hat{y}\neq y)
\;\le\;
C(1-\eta_\phi) + e_\alpha.
\end{equation}
\end{proposition}

\paragraph{Implication.} Proposition~\ref{prop:agree_bound_main} links the hallucination detection error of a supervised probe on \model-shaped embeddings to two terms: a \emph{label-free} agreement-separation quantity $(1-\eta_\phi)$ and $e_\alpha$ that captures how informative stability $\alpha$ is for truthfulness. The first term is exactly what our shaping objective promotes, so improving $\eta_\phi$ tightens the bound. When stability is predictive of correctness (small $e_\alpha$, empirically verified in Section~\ref{sec:ablationstudy}), this implies that stronger agreement separation directly translates into lower hallucination detection error. Full statements and proofs are in Appendix~\ref{app:theory}.

\subsection{Test-time Detection}
\label{sec:testtime}

At test time, given a prompt $\overline{\bx}$, a reasoning trajectory $\overline{\br}$, and a final answer $\overline{\ba}$ produced by the frozen LRM $\pi_\theta$, we extract the {trace-conditioned answer embedding} via $ \overline{\bz} = g_\phi(\overline{\bu}) \in \mathbb{R}^k$, which is then fed to various embedding-based detection scoring functions~\cite{du2024haloscope,burns2022discovering,azaria2023internal} and the scores can be thresholded to obtain a binary prediction.
Importantly, \model does \emph{not} require any sampling at test time; all perturbations are only used to shape the LRM embeddings during training.

\section{Experiments}
\label{scr:experiment}
In this section, we present comprehensive empirical evidence to validate the effectiveness of \model on diverse hallucination detection tasks. Section~\ref{sec:setup} details the setup, and Sections ~\ref{sec:mainresult}--\ref{sec:ablationstudy} provide results and detailed analyses.
\subsection{Setup}
\label{sec:setup}
\paragraph{Datasets and models.}
We evaluate \model on four representative reasoning tasks including both open-domain conversational question-answering (QA) and multi-step mathematical reasoning tasks. We use TruthfulQA~\cite{lin2022truthfulqa}, which contains 817 conversational QA pairs, and TriviaQA~\cite{joshi2017triviaqa}. Following~\cite{lin2023generating}, we utilize the deduplicated validation split of TriviaQA (rc.nocontext subset) with 9,960 QA pairs. For mathematical reasoning tasks, we use GSM8K~\cite{cobbe2021gsm8k} (train split, 7,473 problems) and MATH-500, a curated subset of MATH~\cite{hendrycks2021measuring} with 500 challenging problems. For all datasets, we reserve 25\% of the available data for testing, 100 examples for validation, and the remaining examples for training. By default, greedy decoding is used to generate model answers, though additional sampling strategies are analyzed in Appendix~\ref{apdx:sampling}. Additional experiments on broader reasoning domains are provided in Appendix~\ref{app:bbh-generalization}.

We evaluate our method using two LRM families: Qwen3-8B/14B~\cite{yang2025qwen3technicalreport} and DeepSeek-R1-Distill-Llama-8B/DeepSeek-R1-Distill-Qwen-14B~\cite{deepseekai2025deepseekr1incentivizingreasoningcapability}, which are widely adopted reasoning models with accessible internal representations suitable for intervention and shaping. All models are evaluated in a zero-shot setting using pretrained weights. More dataset and inference details are provided in Appendix~\ref{sec:prompt-app}.

\paragraph{Baselines.} We first compare the vanilla LRM embeddings vs. our trace-conditioned answer representations on four embedding-based detection methods, including supervised probing~\cite{azaria2023internal}, Contrast-Consistent Search (CCS)~\cite{burns2022discovering}, EigenScore~\cite{chen2024inside}, and HaloScope~\cite{du2024haloscope}. Then, we further compare \model with  a comprehensive collection of baselines, including
\emph{uncertainty-based} methods (Perplexity~\cite{ren2022out}, Semantic Entropy~\cite{kuhn2023semantic}), \emph{consistency-based} methods (Lexical Similarity~\cite{lin2023generating}, SelfCKGPT~\cite{manakul2023selfcheckgpt}),
\emph{verbalized} methods (Verbalized Certainty~\cite{lin2022teaching}), and \textit{embedding-steering} method (Truthfulness
Separator Vector (TSV)~\cite{park2025steer}). In addition, we include \emph{methods specifically designed for LRMs} (Reasoning and Answer Consistency Evaluation (RACE)~\cite{wang2025joint}, Reasoning Hallucination Detection (RHD)~\cite{sun2025detection} and graph-based detector (G-Detector)~\cite{anonymous2025unraveling}). To ensure a fair comparison, we assess all baselines on identical test data, employing the default experimental configurations as outlined in their respective papers. We discuss the  details for baselines in Appendix~\ref{sec:details-app}.

\paragraph{Evaluation.}
We evaluate performance with the area under the receiver operating characteristic curve (AUROC), which measures detection performance of a binary classifier under varying thresholds. Correctness labels are obtained using a strong external judge model Qwen3-32B following~\citet{yao2025reasoning}. Answer agreement $\mathrm{Agr}(\cdot,\cdot)$ is judged by the same LRM that generates the trace and answer. Additionally, we show the results remain robust under different agreement measures in Appendix~\ref{sec:label-judge-app}. Additional results under metrics beyond AUROC are provided in Appendix~\ref{app:extra-metrics}.

\paragraph{Implementation Details.} The lightweight mapping $g_{\phi}$ is implemented as a single linear projection without bias, which is optimized using Adam with a learning rate of 1e-4, weight decay of 1e-5, cosine learning rate decay, and a batch size of 128. Following~\citet{azaria2023internal}, we extract the last-token embedding of the answer for training the trace-conditioned representation. 
The layer of training \model, output dimension $k$, temperature $\tau$, training epochs, number of intervention $M$ and noise scale $\sigma$ are determined using the validation split (ablated in Section~\ref{sec:ablationstudy} and Appendix~\ref{appdx:ablation}). The details of the embedding-based detectors (supervised probing, CCS, EigenScore, and HaloScope) are provided in Appendix~\ref{sec:details-app}.

\subsection{Main Results}
\label{sec:mainresult}

\paragraph{Effect of answer-agreement representation shaping.} Table~\ref{tab:orig-vs-proj} compares hallucination detection using \emph{vanilla} LRM embeddings versus our \emph{\model-shaped} trace-conditioned answer representations across four representative embedding-based detectors. Across datasets and both LRMs, shaping consistently yields large gains, indicating that \model improves the \emph{intrinsic separability} of the shaped representations rather than relying on any specific scoring rule: on Qwen3-8B, CCS jumps from $66.85\!\rightarrow\!86.64$ on TruthfulQA and $59.24\!\rightarrow\!88.54$ on TriviaQA, while HaloScope improves from $57.78\!\rightarrow\!71.03$ on TruthfulQA and from $55.73\!\rightarrow\!67.89$ on TriviaQA; similar improvements hold for supervised probing and EigenScore, with especially large gains in challenging regimes such as supervised probing on MATH-500 ($67.03\!\rightarrow\!78.17$). Importantly, the trend also transfers to DeepSeek-R1-Distill-Llama-8B.

\begin{table}[t]
\centering
\caption{\small 
\textbf{Comparison of detecting hallucination using the vanilla LRM embeddings vs. our \model-shaped representations}.
All values are percentages (AUROC) on different embedding-based detection methods. The best results are highlighted in \textbf{bold}.
}
\label{tab:orig-vs-proj}
\tabcolsep 0.04in\renewcommand\arraystretch{0.745}{\small {}}%
\scalebox{0.5}{
\begin{tabular}{l l cc cc cc cc}
\toprule
\multirow{2}{*}{\textbf{Model}} & \multirow{2}{*}{\textbf{Dataset}} 
& \multicolumn{2}{c}{\makecell{\textbf{CCS}\\\cite{burns2022discovering}}}
& \multicolumn{2}{c}{\makecell{\textbf{Supervised Probing}\\\cite{azaria2023internal}}}
& \multicolumn{2}{c}{\makecell{\textbf{HaloScope}\\\cite{du2024haloscope}}}
& \multicolumn{2}{c}{\makecell{\textbf{EigenScore}\\\cite{chen2024inside}}} \\
\cmidrule(lr){3-4} \cmidrule(lr){5-6} \cmidrule(lr){7-8} \cmidrule(lr){9-10}
& & Vanilla & Shaped & Vanilla & Shaped & Vanilla & Shaped & Vanilla & Shaped \\
\midrule

\multirow{4}{*}{Qwen3-8B}
& TruthfulQA & 66.85 & \textbf{86.64} & 78.66 & \textbf{83.66} & 57.78 & \textbf{71.03} & 55.78 & \textbf{73.75} \\
& TriviaQA   &59.24     & \textbf{88.54}             & 86.64    & \textbf{91.62}             &55.73  & \textbf{67.89}            & 55.26    & \textbf{56.09}            \\
& GSM8K      & 57.98    & \textbf{90.37}             & 77.88    & \textbf{89.88}             & 66.24 & \textbf{77.49}            & 63.40    & \textbf{63.54}            \\
& MATH-500   &  55.64   & \textbf{78.66}             & 67.03    & \textbf{78.17}             & 59.81 & \textbf{66.79}            & 81.38    & \textbf{83.54}            \\
\midrule

\multirow{4}{*}{\shortstack{DeepSeek-R1-\\Distill-Llama-8B}}
& TruthfulQA & 59.62    & \textbf{80.89}             & 69.40    & \textbf{76.98}             & 58.10 & \textbf{66.59}            & 61.21    & \textbf{67.87}            \\
& TriviaQA   & 63.99    & \textbf{88.86}             & 48.83    & \textbf{87.45}             & 56.61 & \textbf{64.13}            &53.58     & \textbf{54.13}            \\
& GSM8K      & 53.30    & \textbf{74.72}             & 72.62    & \textbf{77.62}             & 55.16 & \textbf{58.62}            & 52.98    & \textbf{55.97}            \\
& MATH-500   & 54.44    & \textbf{86.38}             & 68.92    & \textbf{79.95}             & 63.77 & \textbf{73.34}            & 75.89    & \textbf{79.46}            \\
\bottomrule
\end{tabular}
}
\end{table}

\begin{figure*}[!t]
    \centering

    \begin{subfigure}[t]{0.24\textwidth}
        \centering
        \includegraphics[width=\linewidth]{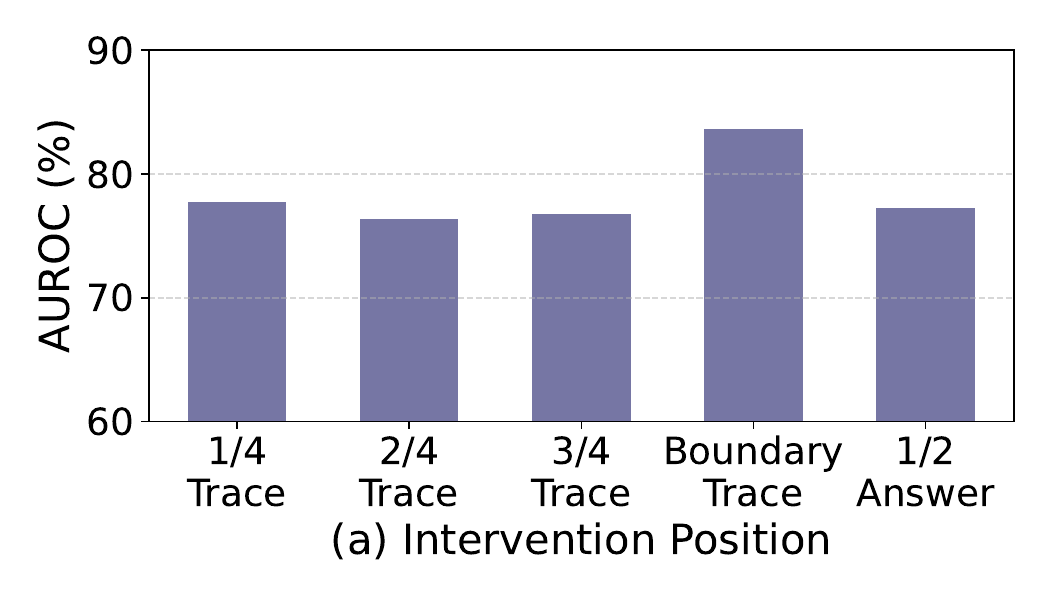}
        \label{fig:noise-position}
    \end{subfigure}
    \hfill
    \begin{subfigure}[t]{0.24\textwidth}
        \centering
        \includegraphics[width=\linewidth]{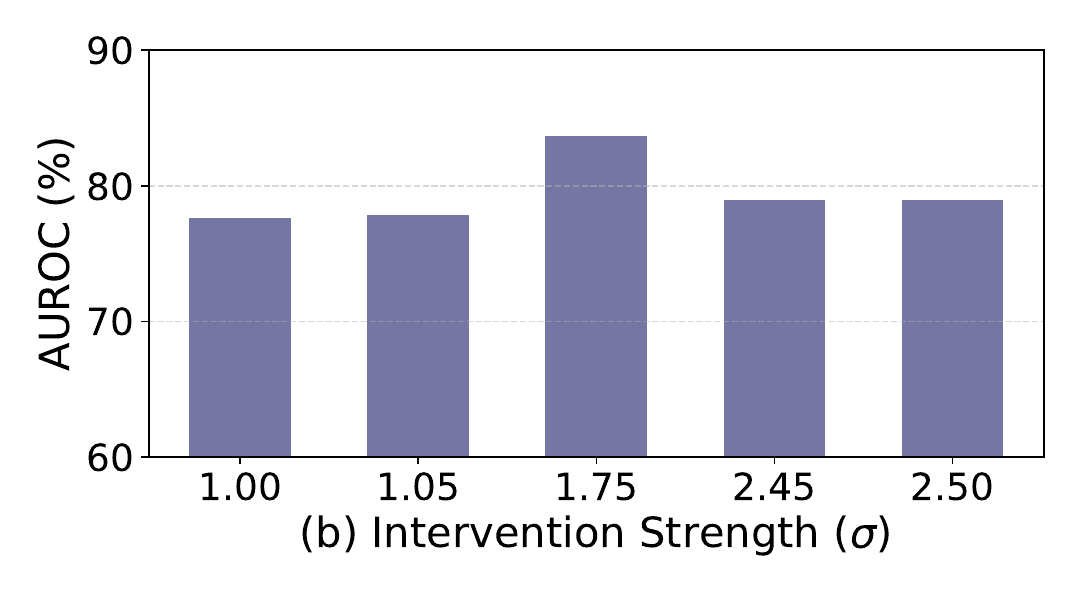}
        \label{fig:noise-alpha}
    \end{subfigure}
    \hfill
    \begin{subfigure}[t]{0.24\textwidth}
        \centering
        \includegraphics[width=\linewidth]{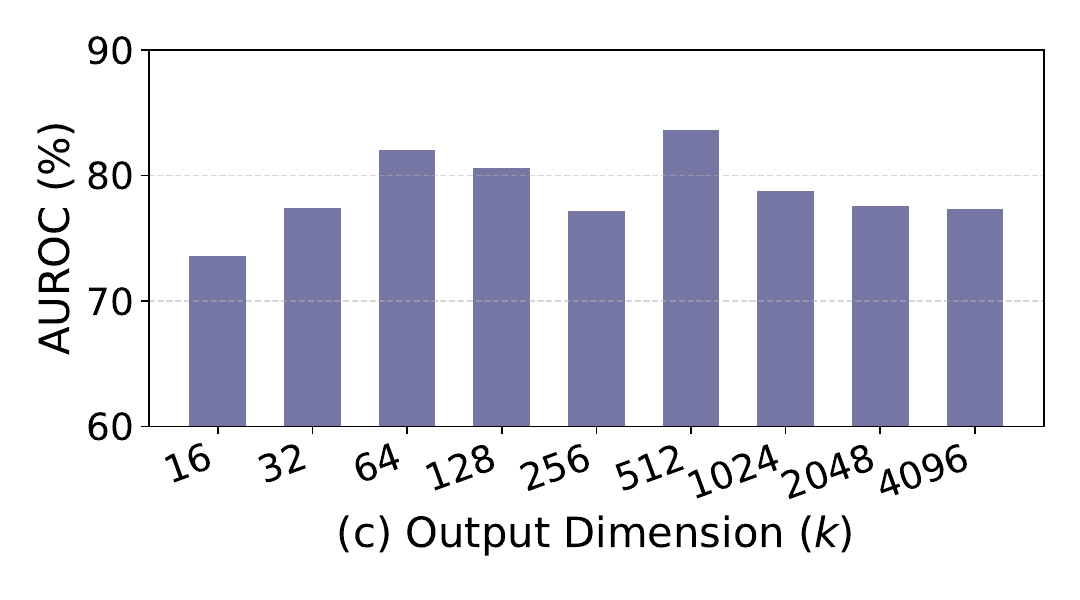}
        \label{fig:sub_c}
    \end{subfigure}
    \hfill
    \begin{subfigure}[t]{0.24\textwidth}
        \centering
        \includegraphics[width=\linewidth]{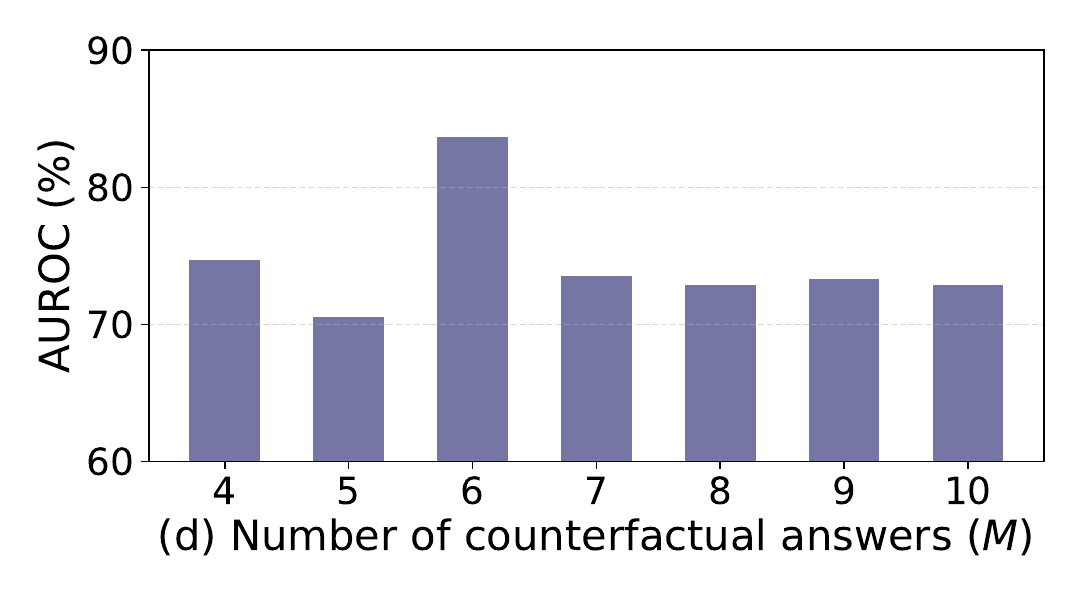}
        \label{fig:sub_d}
    \end{subfigure}
    \vspace{-2em}
    \caption{
    (a) Effect of intervention position,  (b) effect of intervention strength $\sigma$, (c) effect of output dimension $k$ for the trace-conditioned answer representations, and (d) effect of number of counterfactual answers $M$. All results are reported on TruthfulQA using Qwen3-8B. The downstream detector is probing.
    }
    \label{fig:four_horizontal}
\end{figure*}
\paragraph{Comparison with competitive detection baselines.}
Table~\ref{tab:all-detectors} compares \model against a broad suite of detection baselines, including logit-based uncertainty, self-consistency methods that require multiple generations, verbalized confidence prompting, and recent LRM-specific detectors. \model achieves the strongest performance, delivering a large absolute gain on, e.g., TruthfulQA under Qwen3-8B (AUROC 86.64\%), substantially surpassing both general-purpose baselines  and detectors tailored to LRMs, which remain near the low-70s. Notably, \model also outperforms TSV, a semi-supervised representation steering approach, while requiring \emph{zero} human annotations by deriving training signals solely from answer agreement under latent interventions. Beyond accuracy, \model can be computationally practical during inference: unlike sampling-based methods that incur multiple forward passes, \model can directly get the shaped embedding of a \emph{single} model generation for downstream hallucination scoring. The results on a different dataset split are presented in Appendix~\ref{sec:diff-splita-app}. Computational cost is analyzed in Appendix~\ref{sec:comp-time-app}.

\begin{figure}[!b]
  \vspace{-1.6em}
   \begin{subfigure}[t]{0.2\textwidth}
        \centering
        \includegraphics[width=\linewidth]{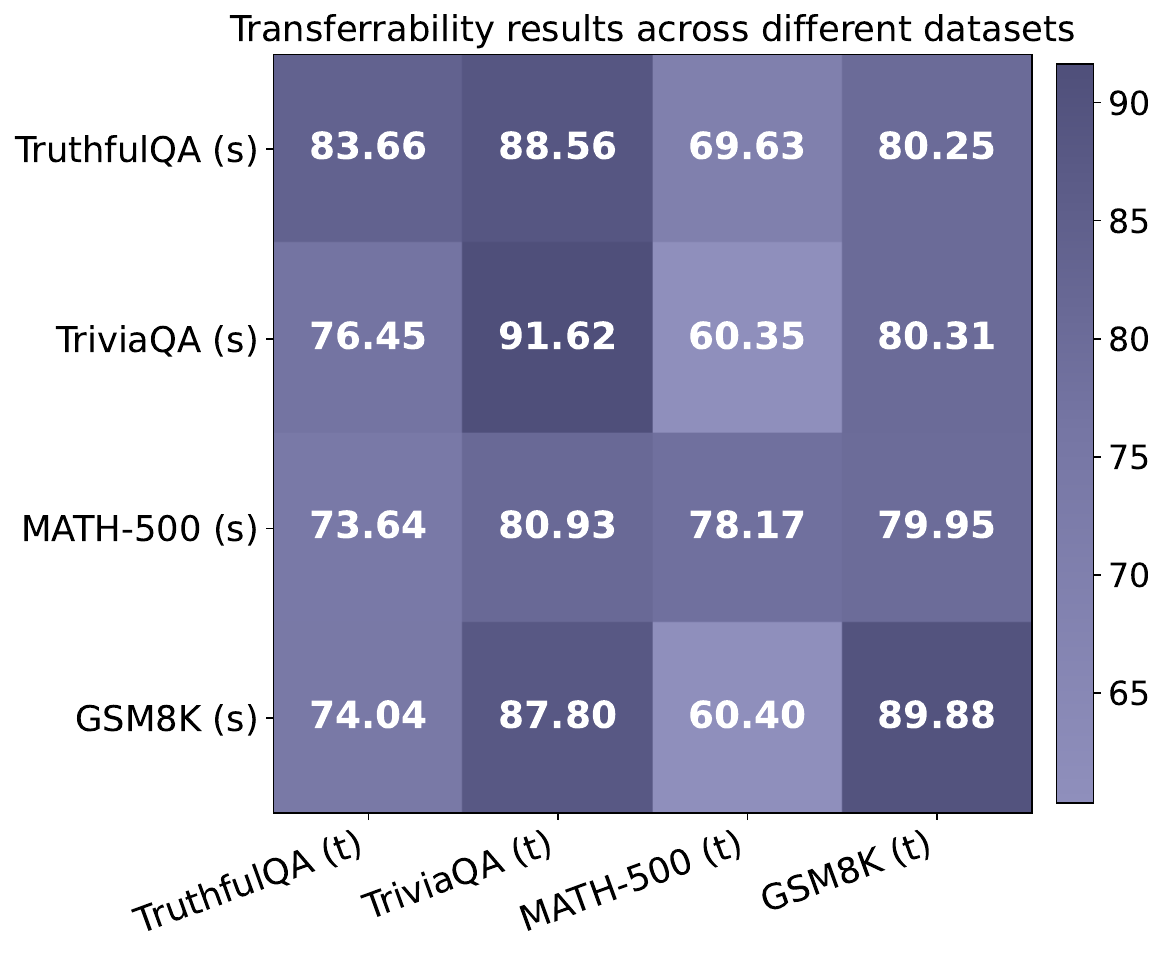}
    \end{subfigure}
    \hfill
     \begin{subfigure}[t]{0.25\textwidth}
        \centering
        \includegraphics[width=\linewidth]{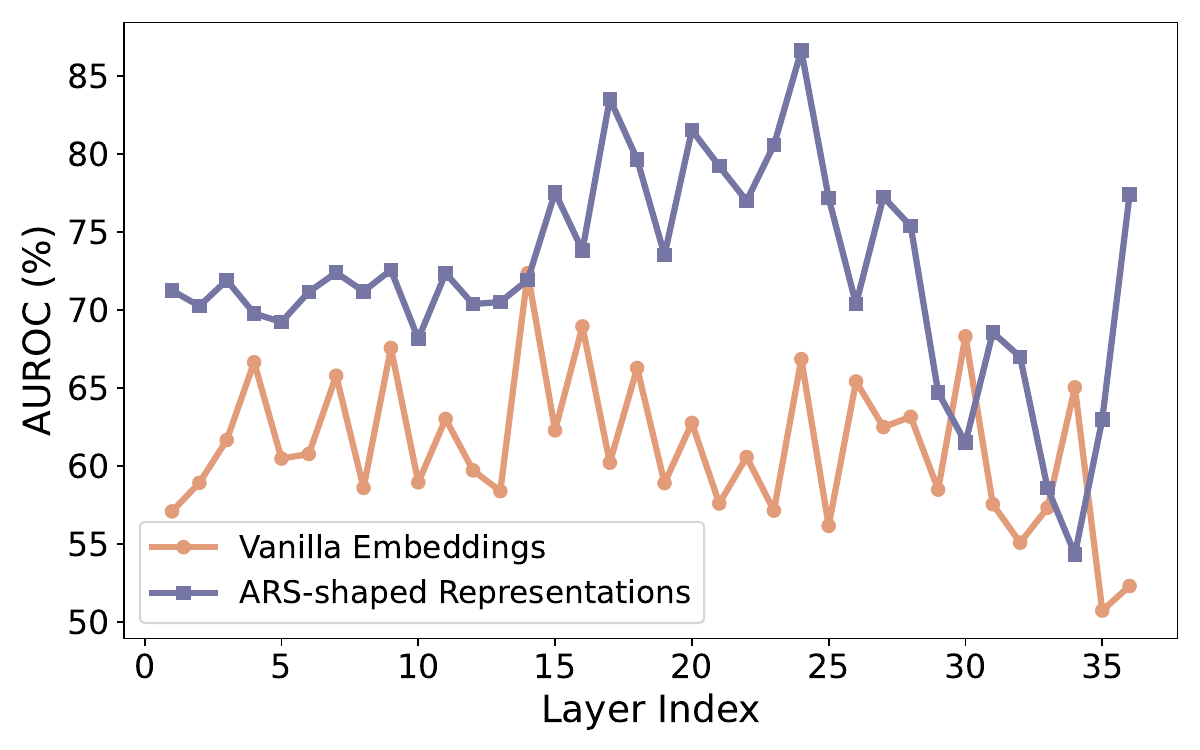}
    \end{subfigure}
    \hfill
  \vspace{-0.7em}
    \caption{\small 
   (\textit{left}) Generalization across datasets, where “(s)” denotes the source data and “(t)” denotes the target data. (\textit{right}) Hallucination detection performance of \model and using vanilla embeddings across different layers (on TruthfulQA). Model used is Qwen3-8B for both (\textit{left}) and (\textit{right}).
    }
    \label{fig:confusion-matrix}
\end{figure}

\subsection{Generalizability and Scalability of \model}
\label{sec:robustness}

\paragraph{Can \model generalize across data distributions?} Following~\citet{park2025steer}, we evaluate whether the embeddings shaped on a \emph{source} dataset $s$ remains effective under distribution shift. Concretely, we train \model on $s$ and apply the learned lightweight mapping $g_{\phi}$ to a different \emph{target} dataset $t$ for downstream detection. As shown in Figure~\ref{fig:confusion-matrix} (\textit{left}), \model transfers reliably across diverse datasets: for example, learning on GSM8K and testing on TriviaQA achieves 87.80\% AUROC, approaching the in-domain result obtained when learning directly on TriviaQA (91.62\%, \textit{probed} on the shaped embeddings). These results indicate that \model captures a stability-driven signal that is largely invariant to dataset-specific surface form, enabling robust hallucination detection for LRMs even under domain shifts.

\begin{table}[!t]
\centering
\caption{ \small Hallucination detection results on larger LRMs. All results are reported based on supervised probing.}
\label{tab:detector-comparison}
\scalebox{0.7}{\begin{tabular}{lcc|cc}
\toprule
\multirow{2}{*}{\textbf{Method}} 
& \multicolumn{2}{c}{\textbf{Qwen3-14B}} 
& \multicolumn{2}{c}{\textbf{DeepSeek-R1-Distill-Qwen-14B}} \\
\cmidrule(lr){2-3} \cmidrule(lr){4-5}
& TruthfulQA & MATH-500 & TruthfulQA & MATH-500 \\
\midrule
TSV &73.41  &80.58  &76.92  &71.46  \\
G-Detector   &69.89  &74.72  &68.44  &71.41  \\
\rowcolor{gray!10} \textbf{\model (Ours) }      &\textbf{77.47}  &\textbf{84.67}  &\textbf{78.52}  &\textbf{79.67}  \\
\bottomrule
\end{tabular}}
\end{table}

\paragraph{ARS scales effectively to larger LRMs.}
To assess scalability, we further evaluate \model on larger LRMs, including Qwen3-14B and DeepSeek-R1-Distill-Qwen-14B. As shown in Table~\ref{tab:detector-comparison}, \model consistently outperforms the two strongest baselines across settings: on TruthfulQA, \model achieves an AUROC of 77.47\% with Qwen3-14B, improving over TSV by 4.06\%, indicating that the stability cues exposed by answer-agreement shaping remains effective in higher-capacity LRMs. Additional results on broader model coverage are provided in Appendix~\ref{sec:broader-coverage}.

\subsection{A Closer Look at  \model}

\label{sec:ablationstudy}
In this section, we conduct a series of in-depth analyses to understand \model. Additional ablations are  in Appendix~\ref{appdx:ablation}.

\paragraph{Ablation on different 
intervention methods.}
Table~\ref{tab:proj-only} systematically compares alternative intervention strategies for constructing answer-agreeing and answer-disagreeing samples. In addition to our latent intervention at the trace boundary, we consider three text-space interventions on the reasoning trace: token deletion, token masking, and trace paraphrasing. For deletion and masking, we sweep the intervention ratio from 
10\% to 
90\% (step size 
10\%) and report the best-performing setting; for paraphrasing, we prompt the LRM itself to rewrite the trace while preserving semantics (prompts in Appendix~\ref{sec:prompt-app}). 

\begin{table}[!b]
\centering
\caption{
Effect of the intervention strategies on four embedding-based detection methods. All results are reported on TruthfulQA using Qwen3-8B.
}
\label{tab:proj-only}
\resizebox{\linewidth}{!}{
\begin{tabular}{l c c c c}
\toprule
\textbf{Intervention Strategies} 
& \textbf{CCS} 
& \textbf{Probing} 
& \textbf{HaloScope} 
& \textbf{EigenScore} \\
\midrule

Deletion        & 75.25 & 73.43 & 65.54 & 61.44 \\
Masking       & 75.50  & 79.19  & 63.36 & 53.78 \\
Paraphrase    & 50.47  & 49.35  & 50.24 & 49.85 \\
\rowcolor{gray!10} \textbf{\model (Ours) }& \textbf{86.64} & \textbf{83.66}  & \textbf{71.03} & \textbf{73.75}\\
\bottomrule
\end{tabular}
}
\end{table}

\begin{figure*}[t]
    \centering
    \begin{minipage}{0.24\textwidth}
        \centering
        \includegraphics[width=\linewidth]{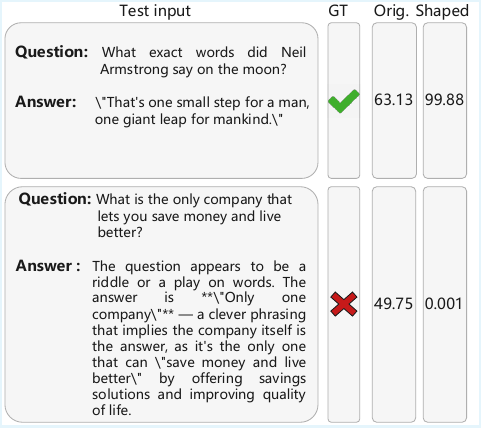}
        \tiny (a) Qualitative examples from TruthfulQA 

    \end{minipage}
    \hfill
    \begin{minipage}{0.32\textwidth}
        \centering
        \includegraphics[width=\linewidth]{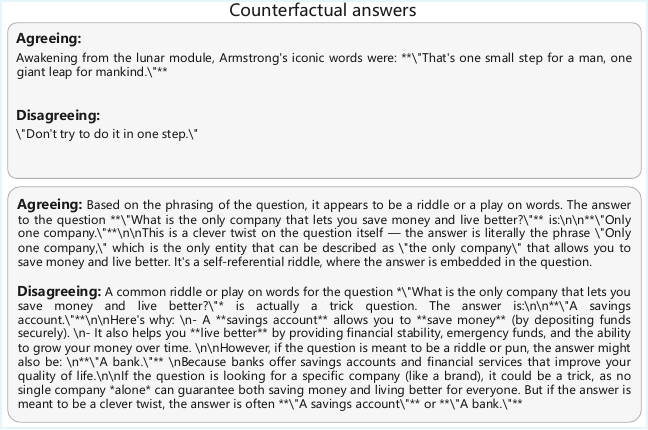}
        \tiny (b) Counterfactual answers for training the lightweight mapping
    \end{minipage}
    \hfill
    \begin{minipage}{0.43\textwidth}
        \centering
        \includegraphics[width=\linewidth]{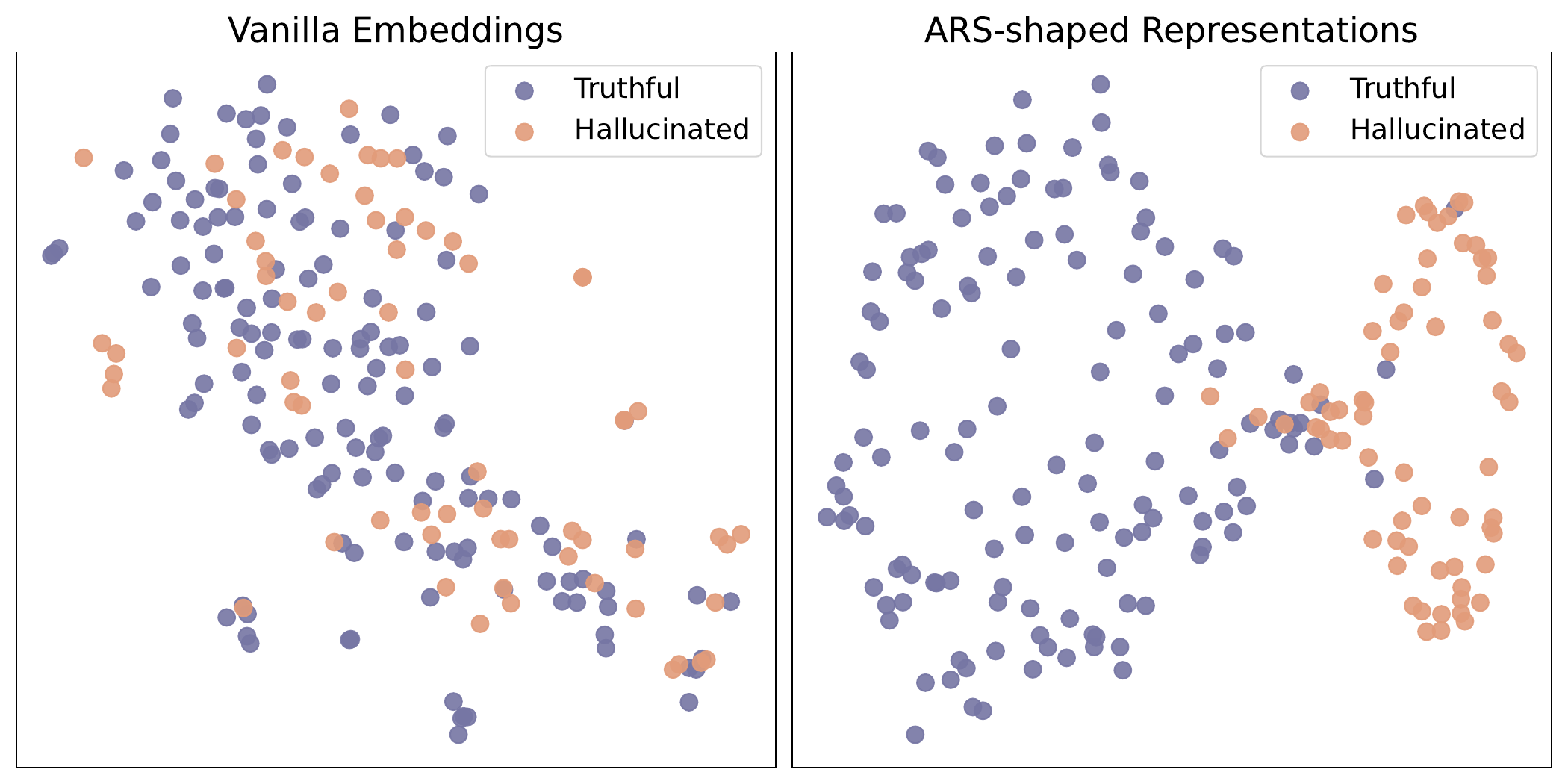}
        \tiny(c) T-SNE visualization of vanilla embeddings and ARS-shaped representations
    \end{minipage}

    \caption{\small 
         (a) Hallucination detection scores of our \model  and using vanilla LRM embeddings (Reasoning trace is omitted for easier presentation. Larger score denotes more truthfulness).
         (b)  Counterfactual answers generated for the examples in (a).
         (c) Embedding visualization comparing vanilla and \model-shaped representation. The model is Qwen3-8B and we utilize questions in TruthfulQA.  
     }
    \label{fig:qualityresult}
\end{figure*}

Overall, \model achieves the strongest performance across all embedding-based detectors. Unlike deletion or masking which often disrupt trace format and coherence in ways that are weakly coupled to answer validity and can introduce spurious cues, our latent perturbations operate directly on the trace boundary, producing counterfactual answers that hold both diversity and relevance. Paraphrasing performs particularly poorly, suggesting that simple rewrites can dilute hallucination-discriminative signals in the produced counterfactual answers, and thus are less useful for representation shaping. In contrast, latent noise injection yields reliable answer agreement/disagreement while minimally altering trace surface form, leading to a more faithful stability signal and substantially improved downstream detection.

\paragraph{How does the intervention position affect performance?}
Moving forward, we further investigate the impact of the position where intervention is applied on overall performance using Qwen3-8B. In Figure~\ref{fig:four_horizontal} (a), we present the effect of  the relative position within the trajectory and answer where the intervention is applied.
We find that 
intervening at the trajectory boundary yields the best performance. Intervention applied too early destabilizes the entire trajectory and reduce the proportion of agreeing answers. Intervening in mid-answer can be constrained by earlier answer tokens, leading to less diverse samples.  This finding empirically validates our design choice of latent intervention (Section~\ref{sec:method1}). 
Additional ablations on interventions within the trace boundary region (e.g., last token vs. last 1\%/2\% tokens) are provided in Appendix~\ref{appdx:ablation}, and sensitivity analyses on perturbation design are reported in Appendix~\ref{app:sensitivity-perturbation}.

\paragraph{Effect of the intervention strength.}
To better understand the characteristics of latent intervention, we vary the intervention strength $\sigma \in \{1.00,1.05,1.75,2.45,2.50\}$ and analyze its effect on the performance, as demonstrated in Figure~\ref{fig:four_horizontal} (b).
The results show that performance improves with the moderate intervention strength (\emph{e.g.}, $\sigma=1.75$), but declines as $\sigma$ increases further. Weak intervention fails to generate sufficiently informative disagreeing answers, while overly strong intervention introduces much easier counterfactual samples, resulting in suboptimal performance. We provide corresponding textual answers in Appendix~\ref{sec:text-exam-sigma-app}.

\paragraph{Effect of the output dimension and number of counterfactual answers.} We study how two key hyperparameters influence \model: the output dimension $k$ of the shaping head $g_\phi$ , and the number of perturbations $M$ used to generate counterfactual answers per example during training (Figure~\ref{fig:four_horizontal} (c)-(d)). Overall, we find that \model is stable across a wide range of settings. Increasing $k$ generally improves performance up to a moderate dimension, after which gains saturate, indicating that the hallucination signal is low-dimensional, which aligns with~\citet{du2024haloscope}. Similarly, a value of 6 for $M$ yields more reliable counterfactual sets and improves detection, but the benefit exhibits diminishing returns beyond it, suggesting that \model can be trained efficiently with limited sampling. Unless otherwise specified, we use $k=\texttt{512}$ and $M=\texttt{6}$ in all experiments.

\paragraph{How do different layers impact \model’s performance?}
In Figure~\ref{fig:confusion-matrix} (\textit{right}), we delve into training \model using embeddings from different layers (evaluated on CCS). All other configurations are kept the same as our main experiments. Consistent with prior findings, intermediate layers provide more discriminative signals than early layers. Notably, \model improves separability across almost {all} layers compared to using the vanilla LRM embeddings, indicating that \model does not rely on a specific layer depth for improvement.

\paragraph{Qualitative results.}
We present qualitative examples of the hallucination detection scores predicted by \model (evaluated on probing) (Figure~\ref{fig:qualityresult} (a)). Leveraging the diverse and informative counterfactual answers (Figure~\ref{fig:qualityresult} (b)), \model can  produce much more separable scores that align with the answer truthfulness  (higher the better). In addition, we also visualize the embeddings of \model and the vanilla LRM in Figure~\ref{fig:qualityresult} (c), where \model shows separable distributions.

\paragraph{Empirical verification of Proposition~\ref{prop:agree_bound_main}.} Proposition~\ref{prop:agree_bound_main} is most informative when the stability score $\alpha$ is predictive of truthfulness (i.e., $e_\alpha$ is small). We verify this by directly using {counterfactual consistency} as the score for detection. Specifically, we compute the thresholding accuracy for hallucination detection, and the result is 80.7\%  on GSM8K dataset when evaluated over all test samples, using Qwen3-8B model. This indicates that stability is indeed informative for truthfulness and thus $e_\alpha$ can be small in practice. This justifies our theoretical analysis.

\section{Conclusion}

We presented \model, a novel framework that leverages reasoning trajectories for hallucination detection by shaping the trace-conditioned answer embeddings. The shaped representations are optimized with answer agreement signals induced by small latent interventions at the trace boundary, which can plug-and-play with existing embedding-based detectors. \model requires no human labels or test-time sampling, and consistently improve detection across datasets, models, and domain shifts. We hope our work will inspire future research
on hallucination detection for long-horizon reasoning LLMs.

\section*{Impact Statement}
This paper develops a method to improve hallucination detection for large reasoning models by leveraging reasoning trajectories and  answer-agreement signals. The primary intended impact is to increase the reliability of deployed LRM systems by better identifying untruthful outputs, which can help mitigate downstream harms such as misinformation, unsafe advice, and overconfident errors in high-stakes applications.  We emphasize that \model is a detection component rather than a guarantee of truthfulness, and should be paired with complementary safeguards (e.g., grounding, policy filters, and human oversight) in real deployments. Overall, our study does not involve human subjects, complies with all legal and ethical standards, and
we do not anticipate any potential harmful consequences
resulting from our work.

\section*{Acknowledgment}
S. Du is supported by NTU start-up grant 025730-00001 and MOE AcRF Tier 1 Seed Funding Grant RS
24/25 025822-00001.
Y. Jiang is supported by National Key R\&D Program of China under Grant No. 2023YFB3308300. B. Guo is supported by National Natural Science Foundation of China under Grant No. U2268204.

\nocite{langley00}

\bibliography{example_paper}
\bibliographystyle{icml2026}

\newpage
\appendix
\onecolumn

\section{Datasets and Implementation Details}
\label{sec:datasets-and-imple-app}
\subsection{Input Prompts}
\label{sec:prompt-app}
We provide the detailed textual input as prompts to the language models for different purposes: (1) generating original reasoning traces and answers (Figure~\ref{fig:prompt_original_generation_qwen} and \ref{fig:prompt_original_generation_deepseek}), (2) evaluating the correctness of the original answers with Qwen3-32B (Figure~\ref{fig:prompt_original_judgment}), (3) generating paraphrased reasoning traces (Figure~\ref{fig:prompt_paragphrase}), (4) generating the counterfactual answers for textual intervention methods (\emph{e.g.}, token deletion and masking, trace paraphrasing) (Figure~\ref{fig:prompt_variant_answer}), and (5) judging the consistency (w/ the LRM itself) between the original answers and their corresponding counterfactual answers (Figure~\ref{fig:prompt_consistency}).
\begin{figure}[H]
    \vspace{-1em}
    \centering
    \includegraphics[width=\linewidth]{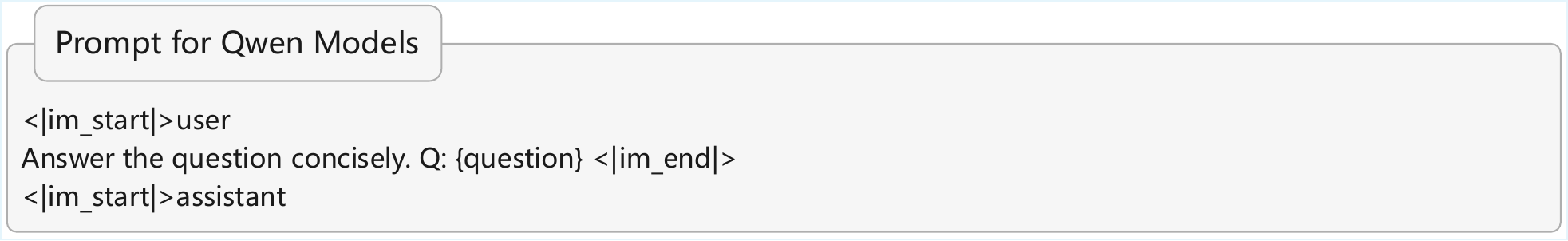}
    \caption{\small Prompt used to generate reasoning traces and answers for Qwen3-8B and Qwen3-14B models. }
    \label{fig:prompt_original_generation_qwen}
    \vspace{-1em}
\end{figure}

\begin{figure}[H]
    \vspace{-1em}
    \centering
    \includegraphics[width=\linewidth]{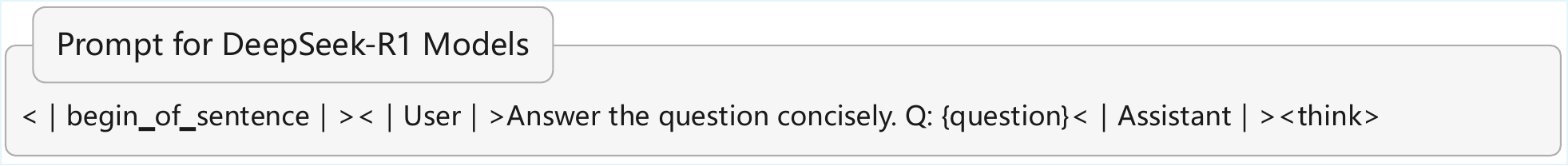}
    \caption{\small Prompt used to generate reasoning traces and answers for DeepSeek-R1-Distill-Llama-8B and DeepSeek-R1-Distill-Qwen-14B models.}
    \label{fig:prompt_original_generation_deepseek}
    \vspace{-1em}
\end{figure}

\begin{figure}[H]
    \vspace{-1em}
    \centering
    \includegraphics[width=\linewidth]{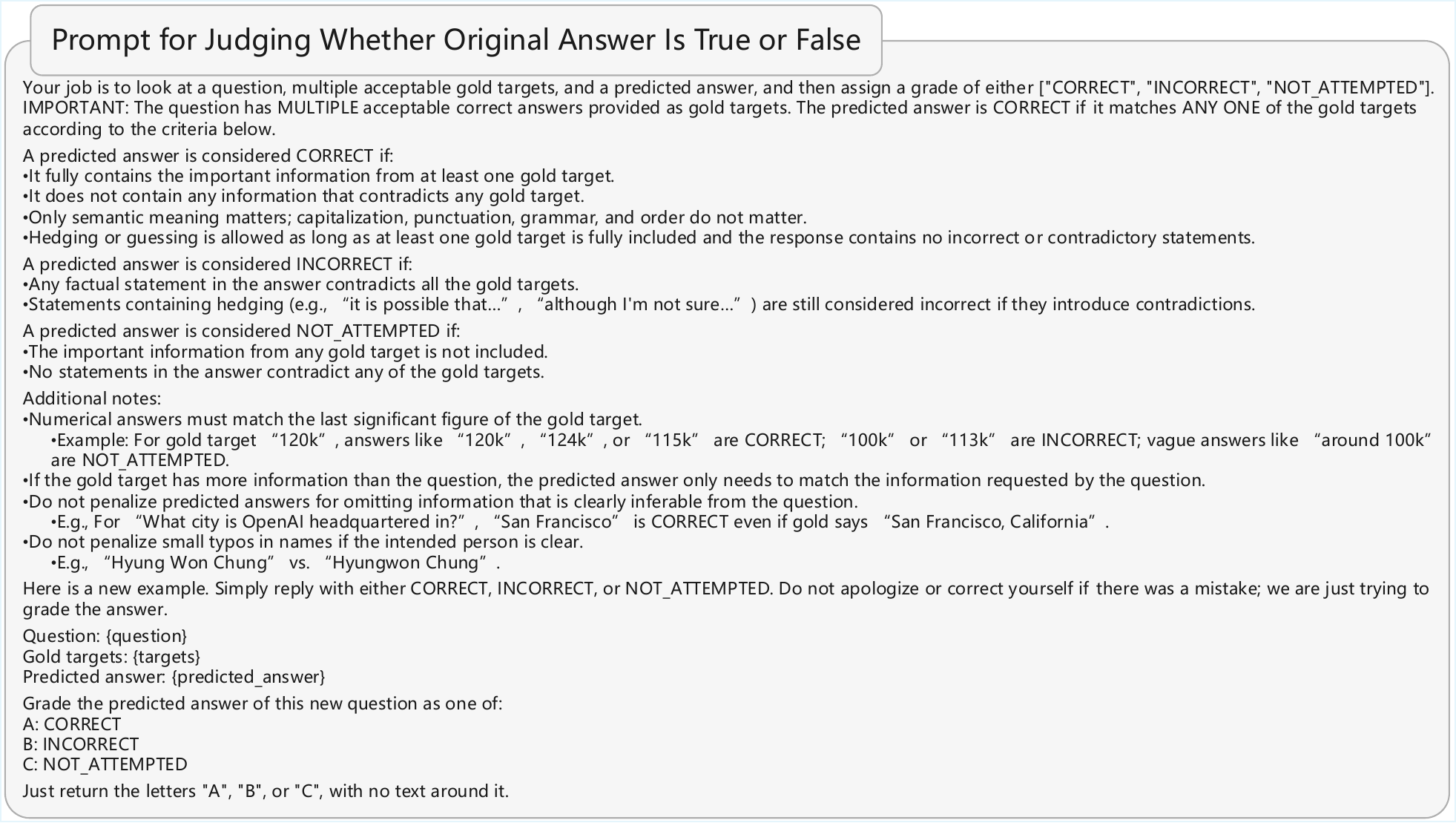}
    \caption{\small Prompt for evaluating the correctness of the original answers (B and C are regarded as hallucinations).}
    \label{fig:prompt_original_judgment}
    \vspace{-1em}
\end{figure}

\begin{figure}[H]
    \vspace{-1em}
    \centering
    \includegraphics[width=\linewidth]{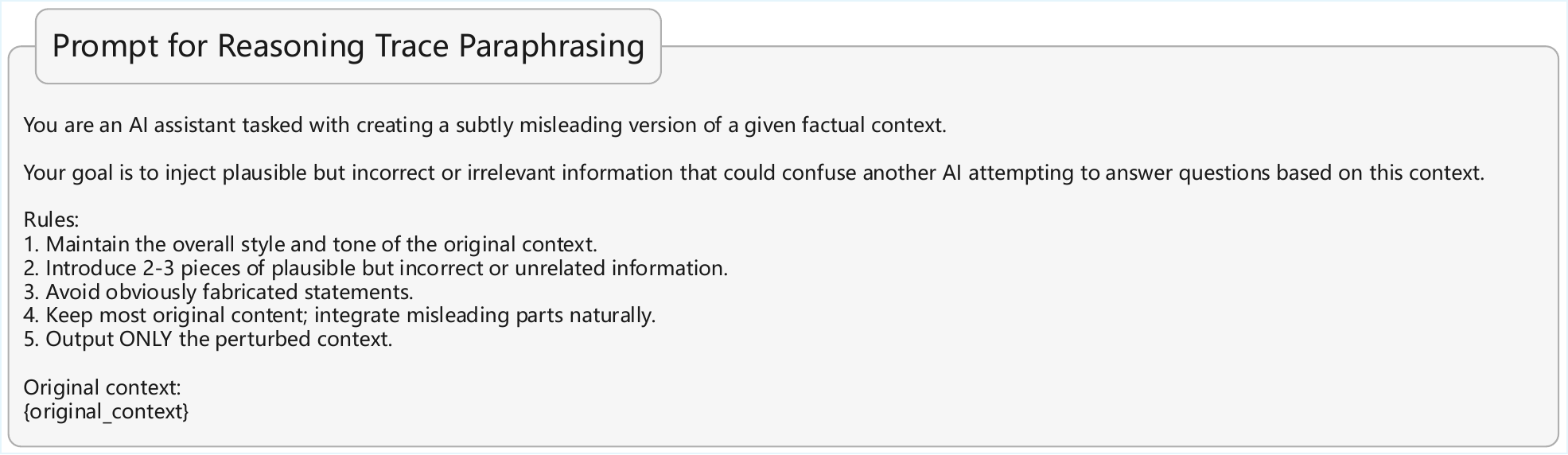}
    \caption{\small Prompt for reasoning trace paraphrasing. We empirically explored many prompting variants and found this paraphrasing with light       information injection can produce reasonably good hallucination detection performance. }
    \label{fig:prompt_paragphrase}
    \vspace{-1em}
\end{figure}

\begin{figure}[H]
    \vspace{-1em}
    \centering
    \includegraphics[width=\linewidth]{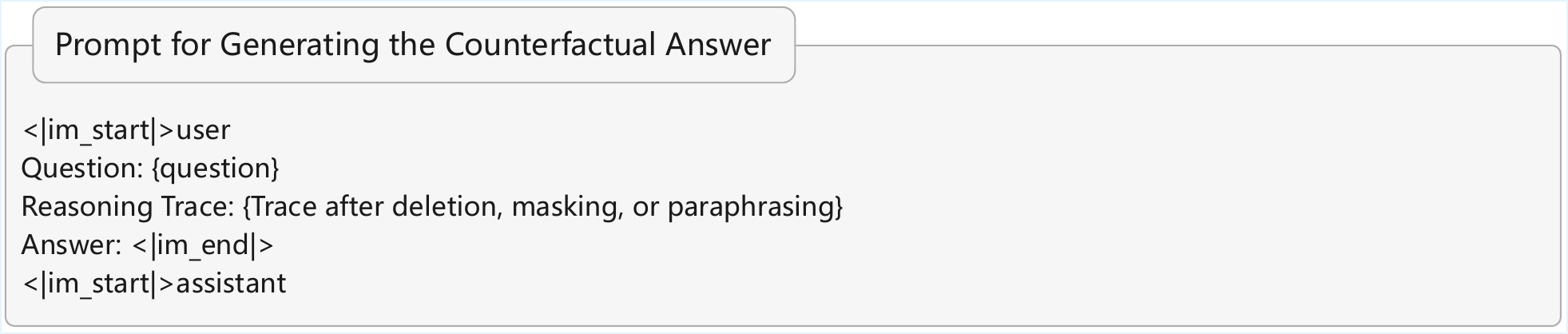}
    \caption{\small Prompt for generating the counterfactual answers in Qwen models (token deletion, token masking and trace paraphrasing).}
    \label{fig:prompt_variant_answer}
    \vspace{-1em}
\end{figure}

\begin{figure}[H]
    \vspace{-1em}
    \centering
    \includegraphics[width=\linewidth]{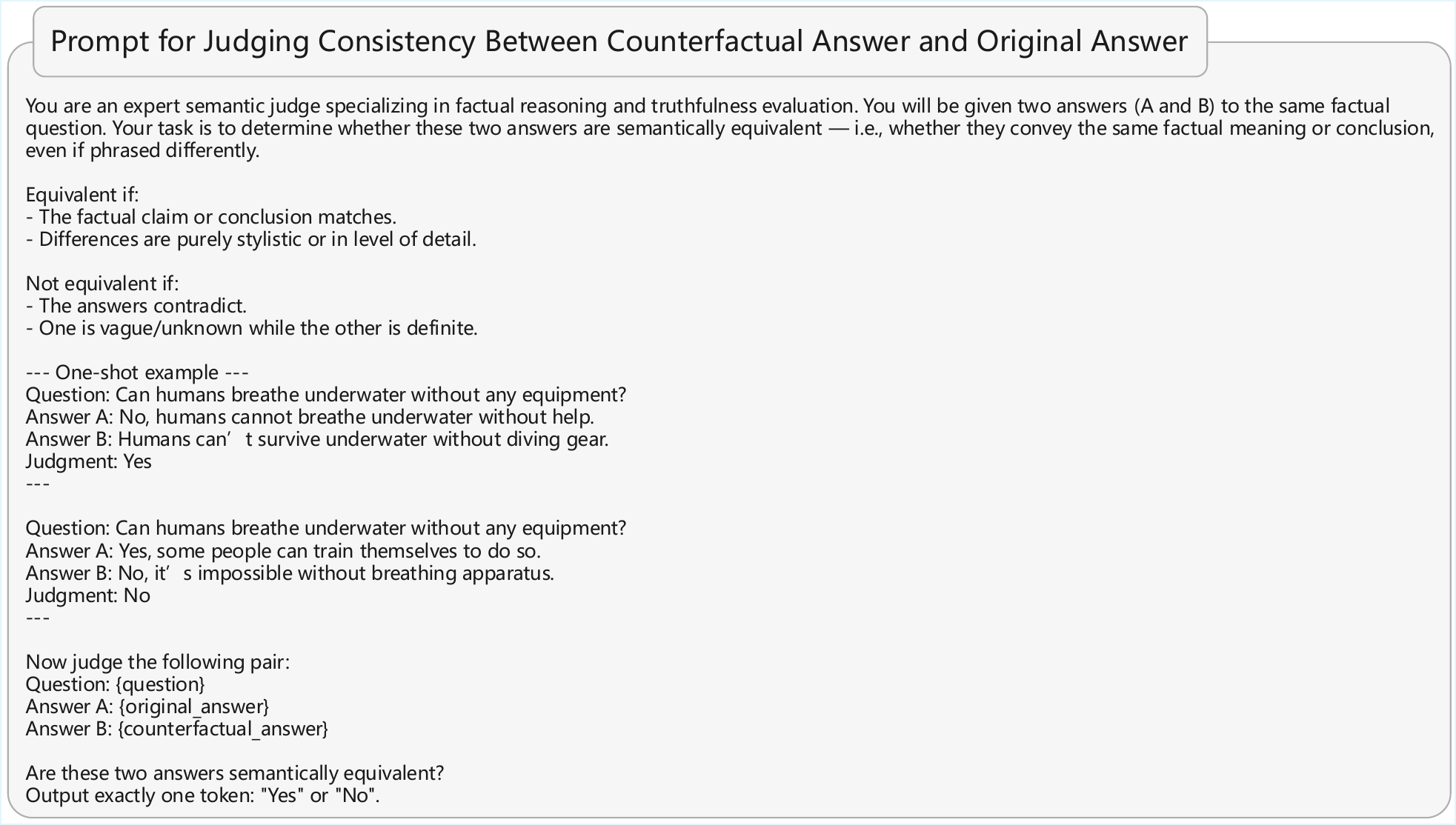}
    \caption{\small Prompt for judging the agreement between the original answers and their corresponding counterfactual answers.}
    \label{fig:prompt_consistency}
    \vspace{-1em}
\end{figure}

\subsection{Additional Implementation Details}
\label{sec:details-app}
\textbf{Supervised probing.}
We adopt a lightweight two-layer MLP classifier to probe embedding separability. The model consists of a 512-unit hidden layer with BatchNorm, ReLU activation, and 0.3 dropout, followed by a logistic output head. We train MLP for 100 epochs with SGD optimizer, an initial learning rate of 0.05, ReduceLROnPlateau learning rate scheduler, batch size of 128, and weight decay of 1e-2. During training, Gaussian noise ($\sigma=0.008$) is injected into the input features for regularization.

\textbf{CCS.}
Following the original paper, we train a lightweight CCS classifier instantiated as a single linear layer, optimized with AdamW (learning rate 1e-3, weight decay 1e-2). The model is trained on balanced positive--negative embedding pairs constructed via difference vectors, using binary cross-entropy for 1000 epochs. We repeat training for 10 random initializations and retain the best-performing checkpoint. Unless otherwise specified, we adopt full-batch training and evaluate performance using AUROC over symmetric test-time difference pairs.

\textbf{EigenScore.}
We evaluate EigenScore on both vanilla and \model-projected embeddings. For each setting, we search the optimal number of sampled generations $K \in [1, 11]$ on a held-out validation split by maximizing AUROC. Using the selected $K$, EigenScore is computed on the test set by averaging log of the eigenvalues for the stacked embeddings from $K$ generations.

\textbf{HaloScope.}
We train HaloScope on the same unlabeled training dataset as \model and  adopt a lightweight MLP. The two-layer MLP has one hidden layer with 512 units, followed by BatchNorm, ReLU, and optional dropout, with Gaussian noise ($\sigma=0.008$) added to input features during training. Classifiers are trained for 50 epochs using SGD with momentum 0.9, initial learning rate 0.05, weight decay 1e-3, batch size 128, and cross entropy loss. The model with the best AUROC on validation is selected. Pseudo-labels for training the MLP are generated via projecting the LRM embeddings to first $k$ principal components. The optimal percentile for getting the pseudo labels, top-$k$ singular vector selection, and sign choice are determined using the separate validation set.

\textbf{Baselines.}
For Perplexity method~\cite{ren2022out}, we follow the implementation details in \citet{cheng2025chain}, and calculate the average perplexity score in terms of the answer tokens. For sampling-based baselines, i.e., Semantic Entropy~\cite{kuhn2023semantic} and Lexical Similarity~\cite{lin2023generating}, we follow the default setting in the original paper and sample 10 generations with a temperature of 0.6 to estimate the uncertainty score. Specifically, for SelfCKGPT~\cite{manakul2023selfcheckgpt}, we adopt the NLI version, which uses a fine-tuned DeBERTa-v3-large model\footnote{\url{https://huggingface.co/MoritzLaurer/DeBERTa-v3-large-mnli-fever-anli-ling-wanli}} to measure the probability of ``entailment" or ``contradiction" between the most-likely generation and the 10 sampled generations. For verbalized method~\cite{lin2022teaching}, we adopt the following prompt for different models. Each choice is mapped to an integer number ($a\mapsto5$ and $f\mapsto0$) for AUROC calculation.
\begin{figure}[H]
    \vspace{-1em}
    \centering
    \includegraphics[width=\linewidth]{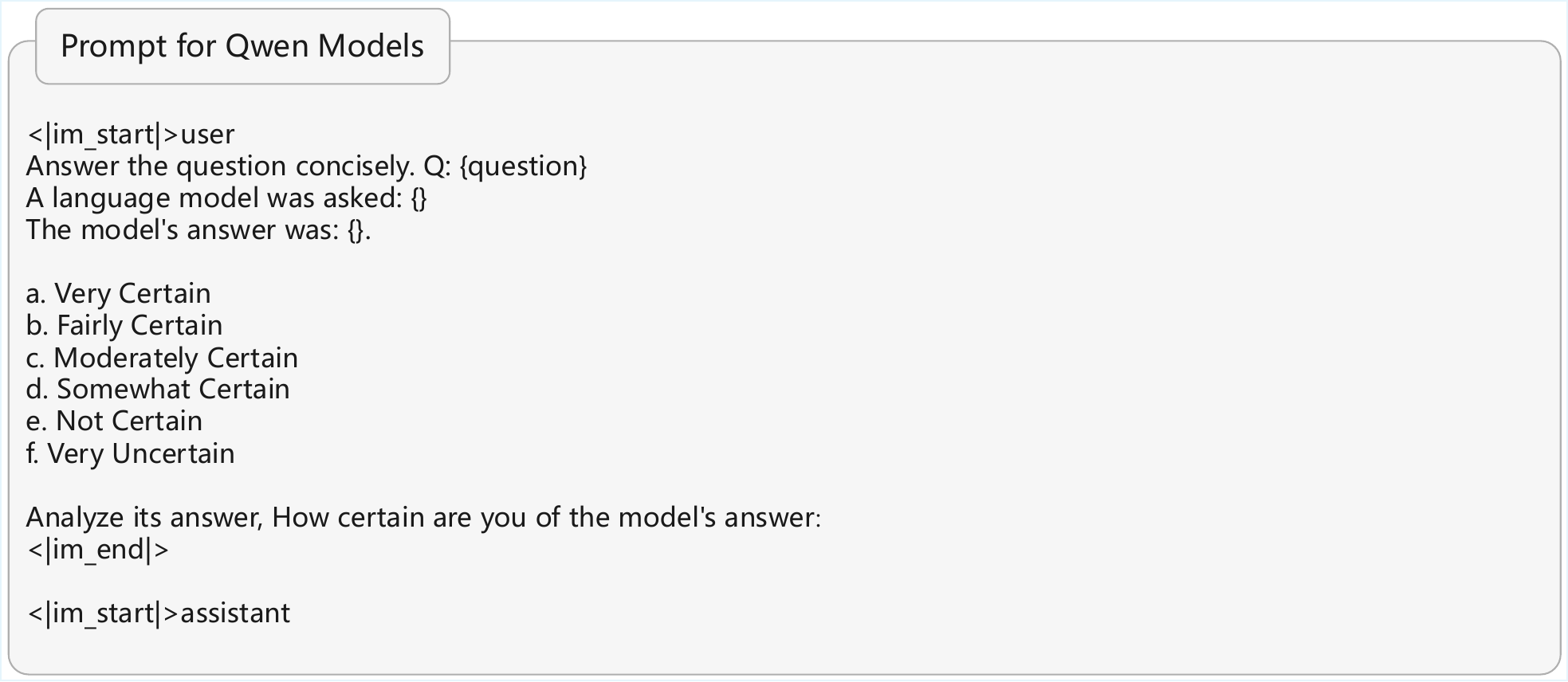}
    \caption{\small Prompt of verbalized certainty baseline~\cite{lin2022teaching} for Qwen models.}
    \label{fig:prompt_verbalized_1}
    \vspace{-1.5em}
\end{figure}

For TSV~\cite{park2025steer}, we follow the default settings described in the original paper, which consist of two stages: (1) the initial training stage and (2) the augmented training stage. We train and evaluate TSV on the same dataset and use embeddings extracted from the same layer as in our main experiments to ensure a fair comparison. For G-Detector~\cite{anonymous2025unraveling}, we follow the default configuration of their released code~\footnote{\url{https://anonymous.4open.science/r/GDetector}} and train a 15-layer GCN with hidden dimension 256, dropout 0.5, global mean pooling, and a two-layer MLP classifier. We also used embeddings extracted from the same layer as in our main experiments to ensure a fair comparison.
For the RACE method~\cite{wang2025joint}, we strictly follow the original implementation and experimental configuration provided by the authors. Concretely, we directly invoke the official RACEScorer API without any modification. We adopt the \textit{Without Pre-Extracted CoTs} setting, where reasoning steps are automatically summarized by RACE using the built-in CoT Extractor. For each sample, six sampled reasoning traces are used to compute the RACE score. For RHD~\cite{sun2025detection}, since there is no code open-sourced, we re-implemented it following the original paper. Specifically, we follow the implementation details in the original paper, candidate reasoning score layers are selected from $\{14, 16, 18, 20, 22, 24, 26\}$ for Qwen3-8B and DeepSeek-R1-Distill-Llama-8B, while attention score layers are fixed across models as $\{1, 3, 5, 7, 9, 11, 13\}$. For Qwen3-8B and DeepSeek-R1-Distill-Llama-8B, we follow the weight settings reported in the original paper, where the weights are set to $\alpha_1 = 0$, $\alpha_2 = 0.9$, $\alpha_3 = 0.8$, and $\alpha_4 = 0.4$ for the Math domain (MATH-500 and GSM8K), and $\alpha_1 = 0$, $\alpha_2 = 0$, $\alpha_3 = 0.3$, and $\alpha_4 = 0$ for the QA dataset. For evaluation, we compute AUROC on the test set using answer-level correctness labels to indicate hallucinations.

\begin{figure}[H]
    \vspace{-1em}
    \centering
    \includegraphics[width=\linewidth]{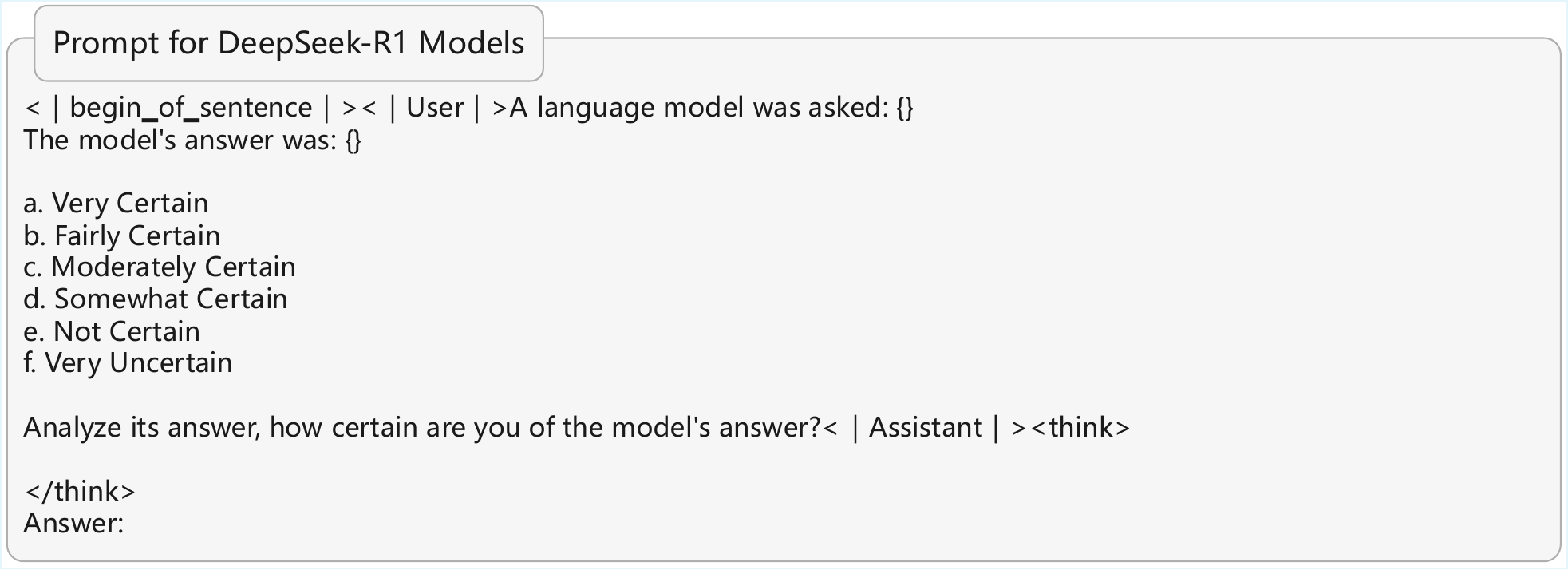}
    \caption{\small Prompt of verbalized certainty baseline~\cite{lin2022teaching} for DeepSeek-R1 models.}
    \label{fig:prompt_verbalized_2}
    \vspace{-1em}
\end{figure}

\begin{figure}[!t]
    \vspace{-0.5em}
    \centering
    \includegraphics[width=\linewidth]{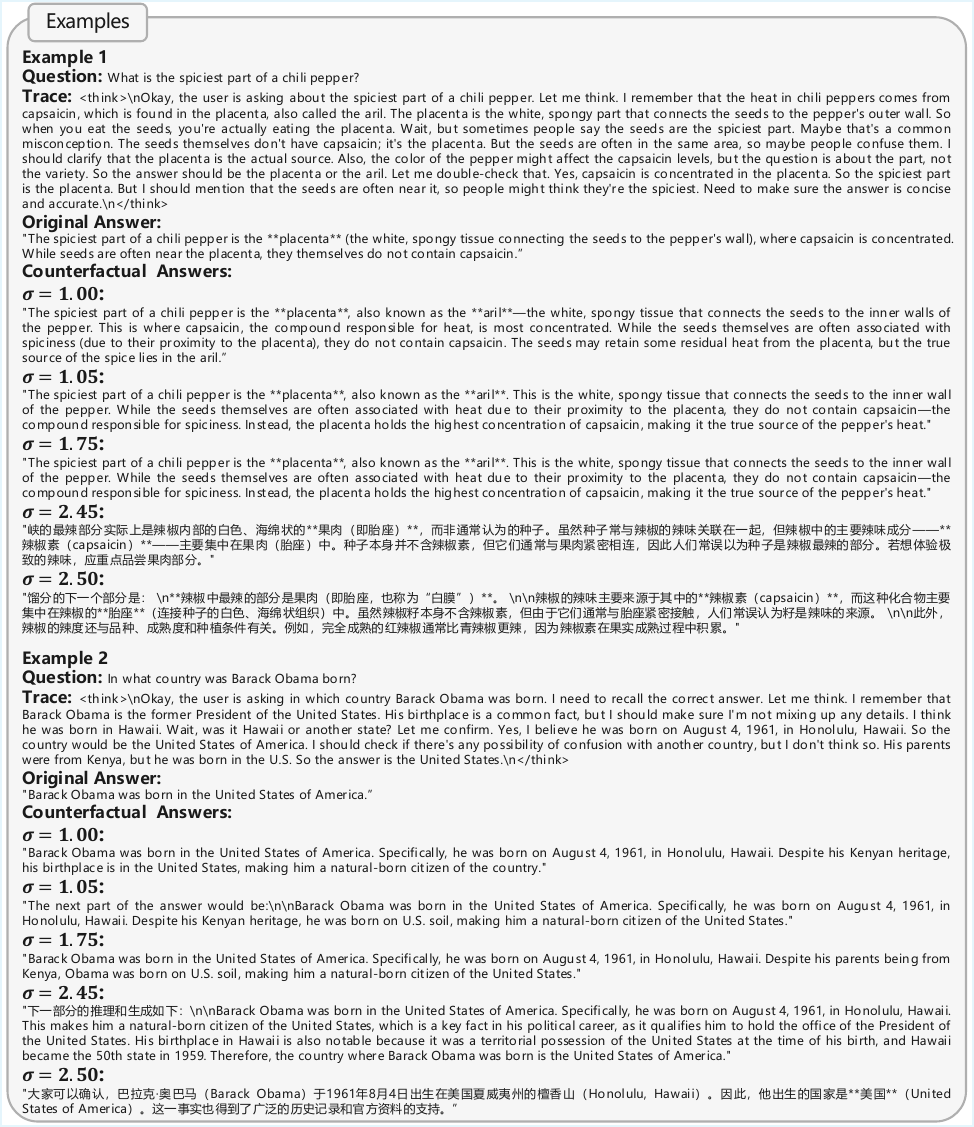}
    \caption{\small Counterfactual answer examples under different intervention strengths.  }
    \label{fig:examples_answer}
    \vspace{-2em}
\end{figure}

\section{Counterfactual Examples with Different Intervention Strengths}
\label{sec:text-exam-sigma-app}
Figure~\ref{fig:examples_answer} showcases the counterfactual answers generated by our \model under different intervention strength, i.e., the standard deviation $\sigma$ of the noise added. A mild value of standard deviation (e.g., $\sigma=1.75$) introduces sufficiently informative variations compared to the original answer but not far away semantically. Larger $\sigma$ can even multi-lingually flip the answer, which makes its contrast with the original answer easy and less significant for shaping the trace-conditioned answer representation. The model is Qwen3-8B and dataset is TruthfulQA.

\section{Ablation on Sampling Strategies}
\label{apdx:sampling}
We evaluate the hallucination detection result when \model generates the original and counterfactual answers under different sampling strategies in inference. In particular, our main results are obtained based on greedy sampling, which generates the next token based on the maximum likelihood. In addition, we compare with multinomial sampling with a temperature of 0.5. The comparison in Table~\ref{tab:sampling-strategies} shows similar performance between the two strategies, with greedy sampling being slightly better.

\begin{table}[H]
\centering
\caption{\small Hallucination detection result under different sampling strategies. Results are based on the Qwen3-8B model and CCS.}
\label{tab:sampling-strategies}
\vspace{-0.5em}
\scalebox{0.85}{
\begin{tabular}{l c c}
\toprule
\textbf{Sampling Strategies} & \textbf{TruthfulQA} & \textbf{MATH-500} \\
\midrule
Multinomial sampling                    &84.49                                &74.95                          \\
\rowcolor{gray!10}
Greedy sampling (Ours)                  &\textbf{86.64}                                &\textbf{78.66}                          \\
\bottomrule
\end{tabular}}
\vspace{-1em}
\end{table}

\section{Results of Using Other Projection Methods}
To verify that the gains of \model do not merely arise from dimensionality reduction, we compare it with two classical projection baselines—PCA and Random Projection (RP). PCA preserves maximal variance in a linear subspace, whereas RP maps embeddings through a data-independent orthogonal matrix.
Across all detectors, \model achieves consistently higher AUROC, while PCA provides moderate improvements and RP exhibits large variance and underperforms. These findings indicate that \model’s advantage stems from its task-aligned objective, which shapes discriminative embeddings rather than relying on unsupervised variance preservation or random subspace selection.

\begin{table}[H]
\centering
\caption{\small
Comparison of four embedding variants (Vanilla, \model, PCA, Rand) across four datasets. CCS is adopted for downstream hallucination detection.}
\label{tab:embed-compare}

\resizebox{0.5\textwidth}{!}{
\begin{tabular}{l l cccc}
\toprule
\textbf{Model} & \textbf{Dataset} & \textbf{Vanilla} & \textbf{\model} & \textbf{PCA} & \textbf{Rand} \\
\midrule

\multirow{4}{*}{Qwen3-8B}
& TruthfulQA 
& 66.85 & \textbf{86.64} &73.26  &68.36  \\

& TriviaQA 
& 59.24 & \textbf{88.54} & 68.10 & 61.70 \\

& GSM8K 
& 57.98 & \textbf{90.37} & 61.54 & 57.71 \\

& MATH-500
& 55.64 & \textbf{78.66} & 60.11 & 54.70 \\
\bottomrule
\end{tabular}
}
\end{table}

\section{Results with a Different Dataset Split}
\label{sec:diff-splita-app}
We verify the performance of our approach using a different random split of the dataset. Consistent with our main experiment, we randomly split 25\% of the datasets for testing using a different seed. \model can achieve similar hallucination detection performance to the results in our main Table~\ref{tab:all-detectors}. For example, on the Qwen3-8B model, our method achieves an AUROC of 87.72\% and 79.43\% under CCS on TruthfulQA and MATH-500 datasets, respectively (Table~\ref{tab:orig-vs-proj-app}). Meanwhile, \model is able to outperform the baselines as well, which shows the statistical significance of our approach (Table~\ref{tab:all-detectors-app}).

\begin{table*}[h]
\centering
\caption{\small
Comparison of the vanilla LRM embeddings vs. our trace-conditioned answer representations with a different random split of the dataset.
All values are percentages (AUROC), and the best results are highlighted in \textbf{bold}.
}
\label{tab:orig-vs-proj-app}

\resizebox{0.7\textwidth}{!}{
\begin{tabular}{l l cc cc cc cc}
\toprule
\multirow{2}{*}{\textbf{Model}} & \multirow{2}{*}{\textbf{Dataset}} 
& \multicolumn{2}{c}{\makecell{\textbf{CCS}\\\cite{burns2022discovering}}}
& \multicolumn{2}{c}{\makecell{\textbf{Supervised Probing}\\\cite{azaria2023internal}}}
& \multicolumn{2}{c}{\makecell{\textbf{HaloScope}\\\cite{du2024haloscope}}}
& \multicolumn{2}{c}{\makecell{\textbf{EigenScore}\\\cite{chen2024inside}}} \\
\cmidrule(lr){3-4} \cmidrule(lr){5-6} \cmidrule(lr){7-8} \cmidrule(lr){9-10}
& & Vanilla & Shaped & Vanilla & Shaped & Vanilla & Shaped & Vanilla & Shaped \\
\midrule

\multirow{2}{*}{Qwen3-8B}
& TruthfulQA & 59.84 & \textbf{87.72} & 76.40 & \textbf{83.65} & 34.17 & \textbf{68.06} & 68.35 & \textbf{70.15} \\
& MATH-500   &51.69  & \textbf{79.43} & 57.71 & \textbf{73.77} & 59.30 & \textbf{69.80} & 77.48 & \textbf{81.38} \\
\bottomrule
\end{tabular}
}
\end{table*}

\begin{table*}[h]
\centering
\caption{\small {Results with a different random split of the dataset.} Comparison with competitive hallucination detection methods on different datasets.  All values are percentages (AUROC). The best results are highlighted in \textbf{bold}.}
\vspace{-0.5em}
\label{tab:all-detectors-app}
\scalebox{0.75}{ 
\begin{tabular}{l l c c c c}
\toprule
\textbf{Model} & \textbf{Method} & \textbf{Single Sampling} &\textbf{Supervision} & \textbf{TruthfulQA} & \textbf{MATH-500} \\
\midrule
\multirow{12}{*}{Qwen3-8B}
& Perplexity~\cite{ren2022out}       & \cmark & \xmark &59.28  &48.52  \\
& Semantic Entropy~\cite{kuhn2023semantic}       & \xmark  & \xmark &60.54  &47.47  \\
& Lexical Similarity~\cite{lin2023generating}    & \xmark   & \xmark &59.22  &46.34  \\
& SelfCKGPT~\cite{manakul2023selfcheckgpt}     & \xmark   & \xmark &50.33  &53.59  \\
& Verbalized Certainty~\cite{lin2022teaching}& \cmark & \xmark &49.88  &45.87  \\
& TSV~\cite{park2025steer}       & \cmark& \cmark  &81.86  &71.40  \\
\cmidrule(lr){2-6}
\multicolumn{6}{c}{\textit{LRM-based}} \\
& RHD~\cite{sun2025detection}   & \cmark & \xmark  &55.38  &48.07  \\
& RACE~\cite{wang2025joint}      & \xmark & \xmark &69.72  &56.83  \\
& G-Detector~\cite{anonymous2025unraveling}       & \cmark & \cmark  &74.15  &58.32  \\

\rowcolor{gray!10}\cellcolor{white}& \textbf{\model (CCS)}    & \cmark & \xmark  & \textbf{87.72} &\textbf{79.43}  \\
\rowcolor{gray!10}\cellcolor{white}& \textbf{\model (Probing)}    & \cmark & \cmark & 83.65 &73.77  \\
\bottomrule
\end{tabular}}
\vspace{-1em}
\end{table*}

\section{Compute Resources and Time}
\label{sec:comp-time-app}
\textbf{Software and hardware.}
We conducted all experiments using Python 3.8.15 and PyTorch 2.3.1~\cite{paszke2019pytorch} on NVIDIA A100 GPUs with 80 GB of memory. 

\textbf{Training and inference time.} Based on tracked runs, the estimated total training and inference time for our approach is notably low. Specifically, on the TruthfulQA training set with the Qwen3-8B model, training \model only takes 649 seconds. During inference on TruthfulQA test set, \model with CCS as detector takes 0.0194 seconds while consistency-based baseline Semantic Entropy requires 1032 seconds, and G-Detector takes 0.2889 seconds to finish. At the same time, \model achieves a much better hallucination detection performance. This highlight the computational efficiency of our approach.

\section{Additional Ablation Studies}
\label{appdx:ablation}

\textbf{Ablation on intervention position.}
We further investigate a finer-grained design choice for intervention: whether intervening on a small window of tokens near the trace boundary (rather than only the last token) affects performance. We compare intervening on the last token with intervening on the last 1\% and 2\% of tokens. As shown in Table~\ref{tab:ablation-tail-region}, intervening at the last several tokens is similarly effective while slightly worse than our choice.
\begin{table}[h]
\centering
\caption{\small
Ablation on perturbation near the trace boundary.
The LRM is Qwen3-8B and dataset is TruthfulQA.
All values are percentages (AUROC). 
}
\label{tab:ablation-tail-region}

\begin{tabular}{lccc}
\toprule
\textbf{Method} & \textbf{Last token (Ours)} & \textbf{Last 1\% tokens} & \textbf{Last 2\% tokens} \\
\midrule
ARS (CCS)      & \textbf{86.64} & 81.70 & 83.32 \\
ARS (Probing)  & \textbf{83.66} & 82.16 & 82.83 \\
\bottomrule
\end{tabular}
\vspace{-1em}

\end{table}

\textbf{Ablation on temperature in the  representation shaping loss.}
We evaluate the effect of the temperature parameter $\tau$ in Equation~\ref{eq:infoNCE} on \model's performance, considering $\tau \in \{0.05, 0.1, 0.5, 1.0, 2.0, 3.0\}$. As shown in Table~\ref{tab:ablation-temperature}, smaller temperatures ($\tau = 0.05$) yield overly sharp distributions of the embedding similarity, which can overemphasize a few positive pairs and reduce robustness. Larger temperatures ($\tau \geq 2.0$) produce overly smooth distributions, weakening the separation between answer-agreeing and answer-disagreeing variants. The optimal performance is observed at $\tau = 0.1$--$1.0$, balancing sensitivity and stability in the embeddings. This ablation confirms that \model  can be affected by temperature during training, but a moderate range suffices for robust hallucination detection.

\begin{table}[H]
\centering
\caption{\small Ablation study on temperature across downstream detectors. Bold numbers indicate the best performance per column. The LRM is Qwen3-8B and dataset is TruthfulQA.}
\label{tab:ablation-temperature}

\resizebox{0.55\textwidth}{!}{
\begin{tabular}{c c c c c}
\toprule
\textbf{Temperature} & \textbf{CCS} & \textbf{Supervised Probing} & \textbf{HaloScope} & \textbf{EigenScore} \\
\midrule
0.05 &84.45 &75.37   & 56.09      &53.64\\
0.1 & \textbf{86.64} & 73.62 & 64.93 & 62.85 \\
0.5 & 63.97 & \textbf{83.66} & \textbf{71.03} & 41.01 \\
1.0 & 66.68 & 73.01 & 64.92 & \textbf{73.75} \\
2.0 & 67.35 & 75.41 & 54.35 & 53.94 \\
3.0 & 54.27 & 78.97 & 62.33 & 60.62 \\
\bottomrule
\end{tabular}
}
\vspace{-1em}
\end{table}

\textbf{Where to extract embeddings from multi-head attention? }
Following~\citet{du2024haloscope,park2025steer}, we investigate the effect of the multi-head attention (MHA) architecture on representing hallucination. Specifically, the MHA can be conceptually expressed as:
\begin{equation}
\label{eq:mha}
\mathbf{\mathrm{f}}_{i+1} = \mathbf{\mathrm{f}}_i + \mathbf{\mathrm{Q}}_i \, \mathrm{Attn}_i\!\left(\mathbf{\mathrm{f}}_i\right)
\end{equation}
where $\mathbf{\mathrm{f}}_i$ represents the output of the $i$-th transformer block, $\mathrm{Attn}_i\!\left(\mathbf{\mathrm{f}}_i\right)$ denotes the output of the self-attention module in the $i$-th block, and $\mathbf{\mathrm{Q}}_i$ is the weight of the feedforward layer. Consequently, we evaluate the hallucination detection performance utilizing embeddings from three \emph{different locations within the MHA architecture}, as delineated in Table~\ref{tab:mha}. The results show that the block output is a favorable choice for detecting hallucinations across both LRM architectures.

\begin{table*}[h]
\vspace{-1em}
\centering
\caption{\small Hallucination detection results on different embeddings locations of multi-head attention. The downstream detector used is CCS. }
\label{tab:mha}
\begin{tabular}{lcc|cc}
\toprule
\multirow{2}{*}{\textbf{Embedding location}} 
& \multicolumn{2}{c}{\textbf{Qwen3-8B}} 
& \multicolumn{2}{c}{\textbf{DeepSeek-R1-Distill-Llama-8B}} \\
\cmidrule(lr){2-3} \cmidrule(lr){4-5}
& TruthfulQA & MATH-500 & TruthfulQA & MATH-500 \\
\midrule
$\mathrm{f}$  &\textbf{86.64}  &\textbf{78.66}  &\textbf{80.89}  &\textbf{86.38}  \\
$\mathrm{Attn(f)}$         &85.17  &67.79  &75.03  &64.17  \\
$\mathrm{\textbf{Q}} \mathrm{Attn(f)}$        &85.64  &71.70  &77.50  &83.75  \\
\bottomrule
\end{tabular}
\vspace{-1em}
\end{table*}

\textbf{Ablation on batch size.}
We study the effect of the batch size on \model's representation learning (Equation~\ref{eq:infoNCE}) by evaluating batch sizes of 16, 32, 64, 128, 256, and 512. Larger batches provide more negative examples per update, which can enhance the separation between answer-preserving and answer-changing variants in the latent space. However, excessively large batches may introduce gradient noise or reduce per-sample learning signal. As shown in Table~\ref{tab:ablation-batchsize}, performance improves substantially from small to moderate batch sizes, peaking around 128--256, after which gains plateau or slightly decrease. This demonstrates that \model achieves robust learning without requiring extremely large batches, balancing computational efficiency and the learning signal quality.

\begin{table}[H]
\vspace{-0.5em}
\centering
\caption{\small Ablation study on batch size across downstream detectors. Best results per column are highlighted in \textbf{bold}.  The LRM
is Qwen3-8B and dataset is TruthfulQA.}
\label{tab:ablation-batchsize}
\resizebox{0.55\textwidth}{!}{
\begin{tabular}{ccccc}
\toprule
\textbf{Batch Size} & \textbf{CCS} & \textbf{Supervised Probing} & \textbf{HaloScope} & \textbf{EigenScore} \\
\midrule
16   & 61.17  &79.52 & 64.76  &54.68  \\
32   & 62.77 & 74.71 & 64.08 & 66.38 \\
64   & 69.41 & 77.51 & 63.14 & 43.90 \\
128  & \textbf{86.64} & \textbf{83.66} & 59.17 & \textbf{73.75} \\
256  & 83.62 & 78.01 & \textbf{71.03} & 54.28 \\
512  & 77.37 & 74.47 & 67.02 & 59.16 \\
\bottomrule
\end{tabular}
}
\vspace{-1em}
\end{table}

\textbf{Ablation on training epochs.}
We investigate the impact of the number of training epochs on \model's performance, considering 50, 100, 200, 300, 500, and 800 epochs. Increasing the number of epochs allows the linear mapping to better align answer-preserving variants and separate answer-changing variants in the shaped embeddings. As shown in Table~\ref{tab:ablation-epochs}, AUROC generally improves with more training, reaching a plateau around 100--300 epochs, after which additional training yields negligible gains. This indicates that \model can effectively learn a robust lightweight projection $g_{\phi}$ with a moderate number of epochs, avoiding overfitting while ensuring sufficient convergence.

\begin{table}[H]
\centering
\caption{Ablation study on number of training epochs across downstream detectors. Best results per column are highlighted in \textbf{bold}. The LRM is
Qwen3-8B and dataset is TruthfulQA.}
\label{tab:ablation-epochs}
\resizebox{0.55\textwidth}{!}{
\begin{tabular}{ccccc}
\toprule
\textbf{Training Epochs} & \textbf{CCS} & \textbf{Supervised Probing} & \textbf{HaloScope} & \textbf{EigenScore} \\
\midrule
50    & 81.26 & 78.78 & 59.12      & 52.08\\
100   & \textbf{86.64} & 78.83 & 61.49 & \textbf{73.75} \\
200   & 83.10 & \textbf{83.66} & 63.69 & 62.85 \\
300   & 81.90 & 72.47 & \textbf{71.03} & 43.23 \\
500   & 81.76 & 74.00 & 47.17 & 52.77 \\
800   & 79.60 & 73.72 & 52.89 & 55.20 \\
\bottomrule
\end{tabular}
}
\vspace{-1em}
\end{table}

\section{Other Judge Method}
\label{sec:label-judge-app}
\textbf{Different method to get ground truth truthfulness label.} In our main paper, the correctness of model generations is judged using a strong external model Qwen3-32B. In this ablation, we show that the results are robust under other different judgment methods, such as Rouge-L, BLEURT and a different LRM, i.e., DeepSeek-R1-Distill-Qwen-32B. Specifically, for ROUGE method, the generation is deemed truthful when the similarity score between the generation and the ground truth exceeds a given threshold of 0.3. In addition, we use the BLEURT metric~\cite{sellam2020bleurt} with the \emph{bleurt-base-128} variant to measure the similarity, a learned metric built upon BERT~\cite{devlin2019bertpretrainingdeepbidirectional} and is augmented with diverse lexical and semantic-level supervision signals.
With the same experimental setup, the results on the DeepSeek-R1-Distill-Llama-8B model are shown in Table~\ref{tab:other_judge}, where the effectiveness of our approach still holds.

\begin{table}[H]
\vspace{-1em}
\centering
\caption{ \small { Main results with Rouge, BLEURT metric and a different LRM, i.e., DeepSeek-R1-Distill-Qwen-32B, to get ground truth truthfulness label.} All values are percentages (AUROC). The best results are highlighted in \textbf{bold}.}
\vspace{-0.5em}
\label{tab:other_judge}
\scalebox{0.73}{\begin{tabular}{l lcc|cc|cc}
\toprule
\multirow{2}{*}{\textbf{Model}} & \multirow{2}{*}{\textbf{Method}} 
& \multicolumn{2}{c}{\textbf{ROUGE}} 
& \multicolumn{2}{c}{\textbf{BLEURT}}
& \multicolumn{2}{c}{\textbf{DeepSeek-R1-Distill-Qwen-32B}} \\
\cmidrule(lr){3-4} \cmidrule(lr){5-6} \cmidrule(lr){7-8}
& & \textbf{TruthfulQA} & \textbf{MATH-500}
  & \textbf{TruthfulQA} & \textbf{MATH-500}
  & \textbf{TruthfulQA} & \textbf{MATH-500} \\
\midrule
\multirow{3}{*}{DeepSeek-R1-Distill-Llama-8B} &TSV~\cite{park2025steer} &91.30  &84.71  &92.93  &\textbf{81.36}  &72.58  &78.64  \\
 & G-Detector~\cite{anonymous2025unraveling} &81.95  &71.48  &73.08  &61.44  &60.24  &38.06  \\

 &\cellcolor{gray!10} \textbf{\model (CCS)}  &\cellcolor{gray!10}\textbf{94.17}  &\cellcolor{gray!10}\textbf{88.00}  &\cellcolor{gray!10}\textbf{93.75} &\cellcolor{gray!10}{81.00}  & \cellcolor{gray!10}\textbf{79.46} &\cellcolor{gray!10}\textbf{86.24}  \\
\bottomrule
\end{tabular}}
\vspace{-1em}
\end{table}

\textbf{Different method to measure answer agreement.}
In our main paper, the agreement between the original answer and its corresponding counterfactual answers is judged by the LRM itself. In this ablation, we show that the results are robust under different judgment methods, such as a different LRM. Specially, on DeepSeek-R1-Distill-Llama-8B model, we use Qwen3-8B to measure agreement. \model achieves a detection AUROC of 80.98\% for TruthfulQA and 87.30\% for MATH-500 (The downstream detector is CCS), which are comparable to the results in Table~\ref{tab:all-detectors}.

\section{Broader Model Coverage}
\label{sec:broader-coverage}
We verify the performance of our approach on Qwen3-4B and 32B models on TruthfulQA dataset as follows, where the effectiveness of our approach still holds (Table~\ref{tab:ars-vs-vanilla-broader-coverage}).

\begin{table}[H]
\centering
\vspace{-0.5em}
\caption{\small
Comparison of the vanilla LRM embeddings vs. our trace-conditioned answer representations on a smaller 4B model and a larger 32B model.
All values are percentages (AUROC), and the best results are highlighted in \textbf{bold}.
}
\label{tab:ars-vs-vanilla-broader-coverage}

\resizebox{0.4\linewidth}{!}{
\begin{tabular}{l cc}
\toprule
\textbf{Method} & \textbf{Qwen3-4B} & \textbf{Qwen3-32B} \\
\midrule
CCS & 58.98 & 51.89 \\
ARS (CCS) & \textbf{83.90} & \textbf{75.96} \\
Supervised Probing & 75.29 & 69.78 \\
ARS (Probing) & 83.64 & 77.85 \\
\bottomrule
\end{tabular}
}
\end{table}

\section{Additional Evaluation Metrics Beyond AUROC}
\label{app:extra-metrics}
To further validate the robustness of ARS across different evaluation metrcis, we additionally report results under Accuracy, F1 score, and AUPRC on the TruthfulQA dataset and Qwen3-8B model. Table~\ref{tab:extra-metrics} shows that ARS can still maintain a considerable improvement over detection on the vanilla embedding space, and a competitive baseline TSV.
\begin{table}[h]
\centering
\caption{\small
Evaluation under additional metrics beyond AUROC on TruthfulQA with Qwen3-8B.
We report F1, AUPRC, and Accuracy.
}
\label{tab:extra-metrics}

\begin{tabular}{lccc}
\toprule
\textbf{Method} & \textbf{F1} & \textbf{AUPRC} & \textbf{Accuracy} \\
\midrule
TSV                  & 64.03 & 68.01 & 74.63 \\
CCS                  & 69.35 & 77.61 & 55.80 \\
ARS (CCS)            & \textbf{74.67} & \textbf{87.60} & 66.47 \\
Supervised Probing   & 57.60 & 55.29 & 67.80 \\
ARS (Probing)        & 71.95 & 75.15 & \textbf{77.56} \\
\bottomrule
\end{tabular}

\end{table}

\section{Sensitivity to Perturbation Design}
\label{app:sensitivity-perturbation}
To further evaluate robustness, we test alternative noise distributions during intervention on the TruthfulQA dataset using the Qwen3-8B model, as shown in Table~\ref{tab:noise-distribution}. We compare our default design (Gaussian distribution) with isotropic uniform sampling and bounded uniform noise. This result show that our design is more effective overall.

\begin{table}[h]
\centering
\caption{\small
Effect of different perturbation distributions on ARS performance.
All values are percentages (AUROC).
}
\label{tab:noise-distribution}

\begin{tabular}{lcc}
\toprule
\textbf{Method} & \textbf{ARS (CCS)} & \textbf{ARS (Probing)} \\
\midrule
Ours                    & \textbf{86.64} & 83.66 \\
Uniform sphere          & 82.12 & \textbf{83.71} \\
Uniform $[-1.75, 1.75]$ & 82.54 & 83.06 \\
\bottomrule
\end{tabular}

\end{table}

\section{Scalability to Broader Reasoning Domains}
\label{app:bbh-generalization}
To evaluate the scalability of ARS across different reasoning domains, we further conduct experiments on the \textit{causal\_judgment} subset of BIG-Bench Hard (BBH) \cite{suzgun2023challenging}, which involves longer and more structured reasoning traces compared to standard factual QA settings. Table~\ref{tab:bbh-results} reports results on the Qwen3-8B model. These results further support that ARS generalizes effectively to broader reasoning domains with longer and more complex reasoning traces, rather than being limited to short-form or domain-specific evaluation settings.

\begin{table}[h]
\centering
\caption{\small
Results on the BBH using Qwen3-8B.
All values are percentages (AUROC).
}
\label{tab:bbh-results}

\begin{tabular}{lc}
\toprule
\textbf{Method} & \textbf{BBH} \\
\midrule
CCS                 & 65.30 \\
ARS (CCS)           & \textbf{79.50} \\
Supervised Probing  & 60.62 \\
ARS (Probing)       & \textbf{83.88} \\
\bottomrule
\end{tabular}

\end{table}


\section{Algorithm Table of \model}

\label{sec:algorithm}

Algorithm~\ref{alg:model} summarizes \model, including (i) the training stage that constructs answer-agreement pairs via latent interventions and learns the shaping map $g_\phi$, and (ii) the test-time stage that applies a chosen embedding-space scoring rule on the shaped embeddings.

\begin{algorithm}[t]
\caption{\model: Answer-agreement Representation Shaping}
\label{alg:model}
\KwIn{Frozen LRM $\pi_\theta$; dataset of LRM generations $\mathcal{S}=\{(\bx,\br,\ba)\}$; agreement function $\mathrm{Agr}(\cdot,\cdot)$; perturbation distribution $\mathcal{D}$; number of perturbations $M$; temperature $\tau$; learning rate $\lambda$; training steps $K$.}
\KwOut{Shaping map $g_\phi:\mathbb{R}^d\to\mathbb{R}^k$ (LRM parameters $\theta$ remain frozen).}

\BlankLine
\textbf{Initialize} parameters $\phi$ of shaping head $g_\phi$.\;

\For{$k=1$ \KwTo $K$}{
    Sample a mini-batch $\mathcal{B}\subset \mathcal{S}$ of triples $(\bx,\br,\ba)$.\;
    $\mathcal{L}\leftarrow 0$.\;

    \ForEach{$(\bx,\br,\ba)\in\mathcal{B}$}{
        \tcp{(1) Extract anchor answer embedding}
        Compute vanilla trace-conditioned answer embedding $\bu \leftarrow \mathrm{Embed}_\theta(\bx,\br,\ba)$.\;
        $\bz \leftarrow g_\phi(\bu)$.\;

        \tcp{(2) Generate counterfactual answers via latent intervention at trace boundary}
        Compute trace-boundary state $\bh \leftarrow h_L(\bx\oplus\br)$ (penultimate layer, last trace token).\;
        Initialize agreement set $\mathcal{U}^+\leftarrow \emptyset$ and disagreement set $\mathcal{U}^-\leftarrow \emptyset$.\;

        \For{$j=1$ \KwTo $M$}{
            Sample perturbation $\boldsymbol{\delta}_j \sim \mathcal{D}$ and form $\tilde{\bh}_j \leftarrow \bh + \boldsymbol{\delta}_j$.\;
            Decode counterfactual answer $\tilde{\ba}_j \leftarrow \mathrm{Decode}_\theta(\bx\oplus\br;\tilde{\bh}_j)$.\;
            Compute counterfactual answer embedding $\tilde{\bu}_j \leftarrow \mathrm{Embed}_\theta(\bx,\br,\tilde{\ba}_j)$.\;
            \eIf{$\mathrm{Agr}(\ba,\tilde{\ba}_j)=1$}{
                $\mathcal{U}^+ \leftarrow \mathcal{U}^+ \cup \{\tilde{\bu}_j\}$\;
            }{
                $\mathcal{U}^- \leftarrow \mathcal{U}^- \cup \{\tilde{\bu}_j\}$\;
            }
        }

        \tcp{(3) Agreement-driven shaping loss}
      
            Sample $\tilde{\bu}^+ \sim \mathcal{U}^+$ and set $\tilde{\bz}^+ \leftarrow g_\phi(\tilde{\bu}^+)$, ${\bz} \leftarrow g_\phi({\bu})$.\;
            Set $\mathcal{Z}^- \leftarrow \{g_\phi(\tilde{\bu}^-):\tilde{\bu}^-\in\mathcal{U}^-\}$\;
            Accumulate loss
            \[
            \mathcal{L}_{\model} \leftarrow \mathcal{L}_{\model}
          -\frac{\mathrm{sim}(\bz,\tilde\bz^+)}{\tau}
+
\log \sum_{\tilde\bz' \in \{\tilde\bz^+\}\cup \mathcal Z^-}
\exp\!\Big(\frac{\mathrm{sim}(\bz,\tilde\bz')}{\tau}\Big)
            \]
        
    }

    \tcp{(4) Batch update shaping head only}
    $\phi \leftarrow \phi - \lambda \nabla_\phi\mathcal{L}_{\model} $.\;
}
\BlankLine
\textbf{Return} $g_\phi$.\;

\end{algorithm}

\clearpage

\section{Theory: Agreement Separation for Hallucination Detection}
\label{app:theory}

This appendix provides a formalization of Proposition~\ref{prop:agree_bound_main}.
It matches the method pipeline: \model learns shaped embeddings $\bz$ using only the answer-agreement signal induced by
latent interventions, and the final hallucination detector is a \emph{supervised probe} trained directly on $(\bz,y)$.
At a high level:
(i) the agreement-separation probability $\eta_\phi$ promoted by our shaping objective controls how well the stability proxy
$\alpha$ can be recovered from $\bz$; and
(ii) when stability is predictive of truthfulness (small $e_\alpha$), a supervised probe on $\bz$ achieves hallucination
detection error close to $e_\alpha$ up to a term depending on $(1-\eta_\phi)$ and probe approximation.

\subsection{Setup and induced agreement distribution}
\label{app:theory:setup}

Let $\pi_\theta$ be a frozen reasoning model and let $(\bx,\br,\ba,y)\sim\mathcal{P}$, where
$y\in\{0,1\}$ indicates whether the final answer $\ba$ is truthful under the task criterion.
Let $\bh := \bh_L(\bx\oplus\br)\in\mathbb{R}^d$ be the trace-boundary representation.

\paragraph{Latent intervention and agreement label.}
Let $\boldsymbol{\delta}\sim\mathcal{D}$ and define the counterfactual answer
$
\tilde{\ba}(\boldsymbol{\delta}) := \mathrm{Decode}_\theta(\bx\oplus\br;\bh+\boldsymbol{\delta}).
$
Let $\mathrm{Agr}(\cdot,\cdot)\in\{0,1\}$ be the answer-agreement indicator (exact match or LRM judge), and define
\begin{equation}
\label{eq:app_agree_label}
A := \mathrm{Agr}\big(\ba,\tilde{\ba}(\boldsymbol{\delta})\big)\in\{0,1\}.
\end{equation}
By Definition~\ref{def:alpha}, the stability score is
$\alpha := \Pr_{\boldsymbol{\delta}\sim\mathcal{D}}[A=1]\in[0,1]$.

\paragraph{Shaped embeddings and similarity scores.}
Let $\bu\in\mathbb{R}^d$ be the vanilla answer embedding for $(\bx\oplus\br,\ba)$ and let
$\tilde{\bu}(\boldsymbol{\delta})$ be the vanilla embedding for $(\bx\oplus\br,\tilde{\ba}(\boldsymbol{\delta}))$.
Given a trained shaping head $g_\phi:\mathbb{R}^d\to\mathbb{R}^k$, define
\begin{equation}
\label{eq:app_z_def}
\bz := g_\phi(\bu),\qquad
\tilde{\bz}(\boldsymbol{\delta}) := g_\phi(\tilde{\bu}(\boldsymbol{\delta})).
\end{equation}
Let $S(\boldsymbol{\delta}) := \mathrm{sim}\!\big(\bz,\tilde{\bz}(\boldsymbol{\delta})\big)$ be cosine similarity.

\paragraph{Conditional positive/negative score distributions.}
Conditioned on the anchor $\bz$, define
\begin{equation}
\label{eq:app_score_posneg}
S^{+} \sim S(\boldsymbol{\delta}) \mid (A=1,\bz),
\qquad
S^{-} \sim S(\boldsymbol{\delta}) \mid (A=0,\bz),
\end{equation}
where $S^{+},S^{-}$ are independent draws given $\bz$.

\subsection{Agreement separation}
\label{app:theory:agree}

Recall the agreement-separation indicator (Eq.~\ref{eq:sep_def}):
$\mathsf{Sep}(\bz,\tilde{\bz}^+,\tilde{\bz}^-)=\mathbf{1}\{S^+\ge S^-\}$.

\begin{definition}[Conditional and marginal agreement separation]
\label{def:eta_cond}
For a given anchor $\bz$, define
\begin{equation}
\label{eq:eta_cond}
\eta_\phi(\bz) := \Pr\!\big[S^{+}\ge S^{-}\mid \bz\big]\in[0,1].
\end{equation}
The marginal agreement-separation probability is
\begin{equation}
\label{eq:eta_marginal}
\eta_\phi := \mathbb{E}\big[\eta_\phi(\bz)\big]
= \Pr\!\big[\mathsf{Sep}(\bz,\tilde{\bz}^+,\tilde{\bz}^-)=1\big].
\end{equation}
\end{definition}

\paragraph{Connection to the shaping objective.}
The objective in Eq.~\ref{eq:infoNCE} increases the relative similarity between an anchor $\bz$ and an agreeing embedding
$\tilde{\bz}^+$ versus disagreeing embeddings $\tilde{\bz}^-$, and thus promotes larger $\eta_\phi$.

\subsection{Agreement separation implies stability is predictable from \texorpdfstring{$\bz$}{z}}
\label{app:theory:alpha_from_z}

The next lemma is a technical bridge for analysis. It shows that if agreement separation is high, then the stability score
$\alpha$ is well-approximated by a function of the single anchor embedding $\bz$ (in expectation).

\begin{lemma}[Existence of a stability surrogate from $\bz$]
\label{lem:alpha_surrogate}
There exists a measurable function $r:\mathbb{R}^k\to[0,1]$ such that
\begin{equation}
\label{eq:alpha_surrogate_bound}
\mathbb{E}\big[|r(\bz)-\alpha|\big] \;\le\; 1-\eta_\phi.
\end{equation}
\end{lemma}

\begin{proof}
Fix an anchor $\bz$.
Define a randomized agreement predictor $\hat A$ for a counterfactual score $S=S(\boldsymbol{\delta})$ as follows.
Independently sample a reference score $R$ from the \emph{opposite} conditional distribution given $\bz$:
\[
R \sim
\begin{cases}
S^- \mid \bz, & \text{if } A=1,\\
S^+ \mid \bz, & \text{if } A=0,
\end{cases}
\]
and output $\hat A := \mathbf{1}\{S\ge R\}$.
This construction is used only to relate agreement separation to an achievable agreement-inference error.

Condition on $\bz$. If $A=1$, then $S\sim S^+\mid\bz$ and $R\sim S^-\mid\bz$ independently, and the error event is $\{S<R\}$:
\[
\Pr(\hat A\neq A \mid A=1,\bz)=\Pr(S^+<S^-\mid\bz)=1-\eta_\phi(\bz).
\]
If $A=0$, then $S\sim S^-\mid\bz$ and $R\sim S^+\mid\bz$ independently, and the error event is $\{S\ge R\}$:
\[
\Pr(\hat A\neq A \mid A=0,\bz)=\Pr(S^-\ge S^+\mid\bz)\le 1-\Pr(S^+>S^-\mid\bz)\le 1-\eta_\phi(\bz).
\]
Therefore, for all $\bz$,
\begin{equation}
\label{eq:agree_err_cond}
\Pr(\hat A\neq A\mid\bz)\le 1-\eta_\phi(\bz),
\qquad\text{and hence}\qquad
\Pr(\hat A\neq A)\le 1-\eta_\phi.
\end{equation}

Now define $r(\bz):=\Pr_{\boldsymbol{\delta}\sim\mathcal{D}}[\hat A=1\mid \bz]$.
For a fixed example $(\bx,\br,\ba)$, let $p:=\Pr_{\delta}(A=1)$ and $\hat p:=\Pr_{\delta}(\hat A=1)$.
By the standard coupling inequality for Bernoulli variables,
\[
|\hat p-p| = |\mathbb{E}_\delta[\hat A-A]|
\le \mathbb{E}_\delta[|\hat A-A|]
= \Pr_\delta(\hat A\neq A).
\]
Taking expectation over $(\bx,\br,\ba,y)\sim\mathcal{P}$ yields
$\mathbb{E}[|r(\bz)-\alpha|]\le \Pr(\hat A\neq A)\le 1-\eta_\phi$, proving~\eqref{eq:alpha_surrogate_bound}.
\end{proof}

\subsection{From stability to hallucination detection with supervised probing}
\label{app:theory:detector}

Define the oracle benchmark using the true stability score:
\begin{equation}
\label{eq:ealpha_def}
e_\alpha := \inf_T \Pr\!\left(\mathbf{1}\{\alpha\ge T\}\neq y\right).
\end{equation}

\begin{assumption}[Threshold regularity of $\alpha$]
\label{ass:margin}
Let $T^\star\in\arg\min_T \Pr(\mathbf{1}\{\alpha\ge T\}\neq y)$ be an optimal threshold.
There exists $B>0$ such that for all $\epsilon\in[0,1]$,
\begin{equation}
\label{eq:alpha_margin}
\Pr\!\left(|\alpha-T^\star|\le \epsilon\right)\le B\epsilon .
\end{equation}
\end{assumption}

\begin{lemma}[Plug-in thresholding with an approximate stability surrogate]
\label{lem:plugin_margin}
Under Assumption~\ref{ass:margin}, for any score $\tilde\alpha\in[0,1]$,
the detector $\hat y=\mathbf{1}\{\tilde\alpha\ge T^\star\}$ satisfies
\begin{equation}
\label{eq:plugin_margin_bound}
\Pr(\hat y\neq y)\;\le\; e_\alpha + B\,\mathbb{E}\big[|\tilde\alpha-\alpha|\big].
\end{equation}
\end{lemma}

\begin{proof}
Let $\hat y^\star=\mathbf{1}\{\alpha\ge T^\star\}$.
Then
\[
\Pr(\hat y\neq y)
\le \Pr(\hat y^\star\neq y) + \Pr(\hat y\neq \hat y^\star)
= e_\alpha + \Pr\!\left(\mathbf{1}\{\tilde\alpha\ge T^\star\}\neq \mathbf{1}\{\alpha\ge T^\star\}\right).
\]
The disagreement event implies $|\alpha-T^\star|\le |\tilde\alpha-\alpha|$.
Condition on $|\tilde\alpha-\alpha|$ and apply Assumption~\ref{ass:margin} to obtain
$\Pr(\hat y\neq \hat y^\star)\le B\,\mathbb{E}[|\tilde\alpha-\alpha|]$.
\end{proof}

\subsection{Proof of Proposition~\ref{prop:agree_bound_main}}
\label{app:theory:proof_prop}

To relate the existence of the thresholding rule $\mathbf{1}\{r(\bz)\ge T^\star\}$ to a \emph{supervised probe},
an explicit approximation term is included below.

\begin{assumption}[Probe approximation]
\label{ass:probe}
Let $\mathcal{H}$ be the hypothesis class used for supervised probing (e.g., linear classifiers or a small MLP).
There exists $h\in\mathcal{H}$ such that
\begin{equation}
\label{eq:probe_eps}
\Pr\!\left(h(\bz)\neq \mathbf{1}\{r(\bz)\ge T^\star\}\right)\le \epsilon_{\mathrm{probe}}.
\end{equation}
\end{assumption}

\begin{proposition}[Agreement separation $\Rightarrow$ bounded error of supervised probing]
\label{prop:agree_bound_formal}
Assume Assumptions~\ref{ass:margin}--\ref{ass:probe}. Let $\eta_\phi$ be defined in Definition~\ref{def:eta_cond}.
There exists a supervised probe $h$ on shaped embeddings $\bz$ such that
\begin{equation}
\label{eq:agree_bound_formal}
\Pr(h(\bz)\neq y) \;\le\; e_\alpha \;+\; B\,(1-\eta_\phi)\;+\;\epsilon_{\mathrm{probe}}.
\end{equation}
\end{proposition}

\begin{proof}
By Lemma~\ref{lem:alpha_surrogate}, there exists $r:\mathbb{R}^k\to[0,1]$ such that
$\mathbb{E}[|r(\bz)-\alpha|]\le 1-\eta_\phi$.
Apply Lemma~\ref{lem:plugin_margin} with $\tilde\alpha=r(\bz)$:
the classifier $\bar y=\mathbf{1}\{r(\bz)\ge T^\star\}$ satisfies
\[
\Pr(\bar y\neq y)\le e_\alpha + B\,\mathbb{E}[|r(\bz)-\alpha|]
\le e_\alpha + B\,(1-\eta_\phi).
\]
Finally, by Assumption~\ref{ass:probe} and a union bound,
\[
\Pr(h(\bz)\neq y)\le \Pr(\bar y\neq y)+\Pr(h(\bz)\neq \bar y)
\le e_\alpha + B\,(1-\eta_\phi)+\epsilon_{\mathrm{probe}}.
\]
\end{proof}

\paragraph{Discussion.}
The bound decomposes hallucination detection error into:
(i) $e_\alpha$, an oracle benchmark reflecting how predictive stability $\alpha$ is of truthfulness on the task;
(ii) a label-free term $(1-\eta_\phi)$ directly promoted by the shaping objective in Eq.~\ref{eq:infoNCE};
and (iii) a probe approximation term $\epsilon_{\mathrm{probe}}$ that depends on the chosen probing class.
This aligns with the method: \model shapes embeddings to increase agreement separation, and the downstream detector is a
supervised probe trained directly on $(\bz,y)$.

\end{document}